\documentclass[twoside,11pt]{article}

\usepackage{keytheorems}
\newkeytheorem{theorem}[style=definition]
\newkeytheorem{definition}[style=definition]
\newkeytheorem{lemma}[style=definition]
\newkeytheorem{observation}[style=definition]
\newkeytheorem{assumption}[style=definition]
\newkeytheorem{proposition}[style=definition]
\newkeytheorem{corollary}[style=definition]

\usepackage{jmlr2e-arxiv}

\jmlrheading{1}{2000}{1-48}{4/00}{10/00}{00-000}{author list}

\usepackage[utf8]{inputenc} 
\usepackage[T1]{fontenc}    
\usepackage{hyperref}       
\usepackage{url}            
\usepackage{booktabs}       
\usepackage{amsfonts}       
\usepackage{nicefrac}       
\usepackage{microtype}      
\usepackage{xcolor}         

\usepackage{amssymb}            
\usepackage{mathtools}          
\usepackage{mathrsfs}           
\usepackage{graphicx}           
\usepackage[space]{grffile}     
\usepackage{url}                
\usepackage{lipsum}             
\usepackage{subfigure}
\usepackage{enumitem}
\usepackage{wrapfig}
\usepackage{lineno}

\usepackage{tikz}
\usetikzlibrary{automata, arrows.meta, positioning, calc}
\usepackage[dvipsnames]{xcolor}

\usepackage{hyperref}
\usepackage{url}
\definecolor{darkblue}{rgb}{0, 0, 0.5}
\hypersetup{colorlinks=true, citecolor=darkblue, linkcolor=darkblue, urlcolor=darkblue}

\usepackage[capitalize,noabbrev]{cleveref} 
\crefname{assumption}{assumption}{assumption}
\Crefname{assumption}{Assumption}{Assumption}

\newcommand{\natnum}{\mathbb{N}}
\newcommand{\realnum}{\mathbb{R}}
\newcommand{\statespace}{\mathcal{S}}
\newcommand{\actionspace}{\mathcal{A}}
\newcommand{\obsspace}{\mathcal{O}}
\newcommand{\histories}{\mathcal{H}}
\newcommand{\bellmanop}{\mathcal{T}}
\newcommand{\bellmanopinf}{\mathcal{B}}
\newcommand{\rmax}{{R_{\max}}}
\newcommand{\vmax}{{V_{\max}}}
\newcommand{\eps}{{\varepsilon}}

\newcommand{\start}{{0}}
\newcommand{\train}{{\text{train}}}
\newcommand{\test}{{\text{test}}}
\newcommand{\obs}{{\text{obs}}}
\newcommand{\unk}{{\text{unk}}}
\newcommand{\est}{{\text{est}}}
\newcommand{\ood}{{\text{\tiny OOD}}}
\newcommand{\defeq}{{\overset{\text{def.}}{~=~}}}

\newcommand{\expec}[2]{{ \mathbb{E}_{#1} \left[ #2 \right]  }}
\newcommand{\probsimplex}[1]{{\Delta ( #1 ) }}
\newcommand{\norm}[1]{{ \| #1 \| }}
\newcommand{\abs}[1]{{ | #1 | }}

\newcommand{\prob}[1]{{ \text{Pr}( #1 ) }}
\newcommand{\ind}[1]{{ \textbf{1}_{[ #1 ]} }}

\ShortHeadings{Smaller Abstract State Spaces Enable Cross-Scale Generalization}{Mustakim and Lehnert}

\begin{document}

\title{Smaller Abstract State Spaces Enable Cross-Scale Generalization in Reinforcement Learning}

\author{\name Nasehatul Mustakim \email nasehatul.mustakim@usask.ca\\
\addr Department of Computer Science\\
University of Saskatchewan\\
Saskatoon, Saskatchewan, Canada
\AND
\name Lucas Lehnert \email lucas.lehnert@usask.ca\\
\addr Department of Computer Science\\
University of Saskatchewan\\
Saskatoon, Saskatchewan, Canada}

\maketitle

\begin{abstract}
While humans readily generalize abstract concepts to more complex or larger tasks, building Reinforcement Learning (RL) systems with this ability remains elusive.
Here, we present the first theoretical model of how such Out-of-Distribution (OOD) generalization can be achieved in RL agents.
Our approach considers Partially Observable Markov Decision Processes (POMDPs) and assumes that an intelligent agent uses an abstraction function to determine which experiences can be treated as equivalent and which must be distinguished. 
First, we extend the existing state abstraction framework and proof techniques to POMDPs. 
Then, we define a successor-weighted model reduction, a model reduction variant that enables compression into smaller abstract spaces than prior definitions allow. 
We derive a bound on the agent's OOD test performance, thereby defining the conditions under which OOD generalization is achievable.
This bound decomposes an agent's performance loss into approximation and estimation errors, revealing how reducing an agent's abstract state space size improves test performance and OOD generalization. 
Our analysis suggests that constraining an agent to operate over a small, finite set of abstract states is necessary for achieving generalization to more complex tasks. 
Our results motivate further research into learning RL architectures that scale across tasks of varying complexity levels.
\end{abstract}


\section{Introduction}
\label{sec:introduction}

Biological intelligence is distinguished by its capacity to generalize learned rules and principles to novel and more complex tasks. 
For instance, after mastering long division with small integers, an elementary school student can apply the same algorithm to larger numbers without additional instruction. 
Here, we present a theoretical model and analysis of how such Out-of-Distribution (OOD) generalization can be achieved in Reinforcement Learning (RL) agents~\citep{sutton2018ReinforcementLearningIntroduction}.
Understanding the computational principles that underlie this level of generalization and abstraction could not only pave the way for more data-efficient and lightweight Artificial Intelligence (AI) systems, but could also provide a deeper understanding of intelligence itself---ranging from insights into human cognition to the design of more effective human–AI interfaces.

\begin{figure}[h!]
    \centering
    \subfigure[
        Warm--cold lattice POMDP. 
    ]{
        \label{fig:warm-cold}
\begin{tikzpicture}[scale=0.6]

    \begin{scope}[shift={(-6.0,3)}]
        \draw[step=0.5 cm, gray, very thin, opacity = 0.1] (-2.45, -2.45) grid (2.45,2.45);
        
        \filldraw[Green] (0,0) circle (3pt);
         \node[below, font = \tiny] at (0,0) {Goal};
        
        \draw[fill=black] (1.5,2.0) circle (2pt);
        \node[left, font = \tiny] at (1.5,2.0) {$s_1$};
        
        \draw[->, thick, Orange, dashed, >=latex] (2,2.5) -- (2,2) node[midway, right, font = \tiny] {W};
        \draw[->, thick, Orange, dashed, >= latex] (2,2) -- (1.5,2)   node[midway, below, font = \tiny] {W};
        
        \draw[fill=black] (1,-.5) circle (2pt);
        \node[below left, font = \tiny] at (1,-.5) {$s_2$};
        
        \draw[->, thick, Cerulean, dashed, >=latex] (1,1) -- (1.5,1) node[midway, below, font = \tiny] {C};
        
        \draw[->, thick, Orange, dashed, >= latex] (1.5,1) -- (1.5,.5)   node[midway, right, font = \tiny] {W};
        \draw[->, thick, Orange, dashed, >=latex]  (1.5,.5) -- (1.5,0) node[midway, right, font = \tiny] {W};
        \draw[->, thick, Cerulean, dashed, >=latex] (1.5,0) -- (1.5,-.5) node[midway, right, font = \tiny] {C};
        \draw[->, thick, Orange, dashed, >=latex]  (1.5,-.5) -- (1,-.5) node[midway, below, font = \tiny] {W};

        \draw[fill=black] (-1.0,.5) circle (2pt);
        \node[right, font = \tiny] at (-1.0,.5) {$s_3$};
        
        \draw[->, thick, Cerulean, dashed ,>=latex] (-1.5,-1.5) arc[start angle=0, end angle=180, radius=0.25cm]  node[midway, below, font=\tiny] {C};
        
        \draw[->, thick, Orange, dashed,>=latex] (-2,-1.5) arc[start angle=180, end angle=360, radius=0.25cm]  node[midway, below, font=\tiny] {W};
        \draw[->, thick, Orange, dashed, >=latex] (-1.5,-1.5) -- (-1.5,-1) node[midway, right, font = \tiny] {W};
        \draw[->, thick, Orange, dashed, >= latex] (-1.5,-1) -- (-1.5,-0.5)   node[midway, right, font = \tiny] {W};
        \draw[->, thick, Orange, dashed, >=latex]  (-1.5,-0.5) -- (-1.5,0) node[midway, right,font = \tiny] {W};
        \draw[->, thick, Cerulean, dashed, >=latex] (-1.5,0) -- (-1.5,0.5) node[midway, right, font = \tiny] {C};
        \draw[->, thick, Orange, dashed, >=latex]  (-1.5,0.5) -- (-1,0.5) node[midway, above, font = \tiny] {W};
    
    \end{scope}
    \begin{scope}[shift={(0, 3)}]
    
        
        \draw[step=1 cm, gray, very thin, opacity = 0.3] (-1.5, -1.5) grid (1.5,1.5);
        
        \filldraw[Green, draw=none]  (0,0) circle (3pt);
        \node[below , font = \tiny] at (0,0) {Goal};
        
        \draw[fill=Orange, draw=none] (0,1) circle (2pt);
        \node[above, font = \tiny] at (0,1) {down};
        \draw[->, thin, Orange, >=latex]  (0,1) -- (0, 0.5);
        
        \draw[fill=Orange, draw=none] (0,-1) circle (2pt);
        \node[below, font = \tiny] at (0,-1) {up};
        \draw[->, thin, Orange, >=latex]  (0,-1) -- (0, -0.5);
        
        \draw[fill=Orange, draw=none] (1,0) circle (2pt);
        \node[right, font = \tiny] at (1,0) {left};
        \draw[->, thin, Orange, >=latex]  (1,0) -- (0.5, 0);
        
        \draw[fill=Orange, draw=none] (-1,0) circle (2pt);
        \node[left, font = \tiny] at (-1,0) {right};
        \draw[->, thin, Orange, >=latex]  (-1,0) -- (-0.5, 0);
        
        \draw[fill=Orange, draw=none] (-1,1) circle (2pt);
        \node[above left, font = \tiny] at (-1,1) {right \& down};
        \draw[->, thin, Orange,  >=latex]  (-1,1) -- (-0.5,1);
        \draw[->, thin, Orange,, >=latex]  (-1,1) -- (-1,0.5);
        
        \draw[fill=Orange, draw=none] (1,1) circle (2pt);
        \node[above right, font = \tiny] at (1,1) {left \& down};
        \draw[->, thin, Orange, >=latex]  (1,1) -- (0.5,1);
        \draw[->, thin, Orange,  >=latex]  (1,1) -- (1,0.5);
        
        
        \draw[fill=Orange, draw=none] (1,-1) circle (2pt);
        \node[below right, font = \tiny] at (1,-1) {left \& up};
        \draw[->, thin, Orange,  >=latex]  (1,-1) -- (1, -0.5);
        \draw[->, thin, Orange,  >=latex]  (1,-1) -- (0.5, -1);
        
        \draw[fill=Orange, draw=none] (-1,-1) circle (2pt);
        \node[below left, font = \tiny] at (-1,-1) {right \& up};
        \draw[->, thin, Orange, >=latex]  (-1,-1) -- (-0.5,-1);
        \draw[->, thin, Orange, >=latex]  (-1,-1) -- (-1,-0.5);
    
    \end{scope}
    
    \draw[->, gray, bend left=30,>=latex] (-7,3.5) to (-3.8,4.4);
    \draw[->, gray, bend right=30, >=latex] (-4.5,2) to (1.1,1.4);
    \draw[->, gray, bend left=35, >=latex] (-4.5,5.2) to (1.6,4.6);

        


\end{tikzpicture} 
    }
    \hspace{0.4in}
    \subfigure[
        Lattice start-state distributions.
    ]{
        \label{fig:warm-cold-start-states}
        \begin{tikzpicture}[scale=0.6]
    \filldraw[Green] (0,0) circle (2.5pt);

    \def\r{3}
    \def\step{0.25}

    \foreach \i in {-3,...,3} {
        \foreach \j in {-3,...,3} {
            \pgfmathtruncatemacro{\cond}{abs(\i)+abs(\j)<=\r && !(\i==0 && \j==0)}
            \ifnum\cond=1
                \filldraw[Cerulean] (\i*\step,\j*\step) circle (1.7 pt);
            \fi
        }
    }

   \def\r{4}
    \def\step{0.25}

    \foreach \i in {-4,...,4} {
        \foreach \j in {-4,...,4} {
            \pgfmathtruncatemacro{\cond}{abs(\i)+abs(\j)==\r && !(\i==0 && \j==0)}
            \ifnum\cond=1
                \filldraw[Magenta] (\i*\step,\j*\step) circle (1.7 pt);
            \fi
        }
    }

    \def\r{5}
    \def\step{0.25}

    \foreach \i in {-5,...,5} {
        \foreach \j in {-5,...,5} {
            \pgfmathtruncatemacro{\cond}{abs(\i)+abs(\j)==\r }
            \ifnum\cond=1
                \filldraw[Peach] (\i*\step,\j*\step) circle (1.7 pt);
            \fi
        }
    }

    \def\r{10}
    \def\step{0.25}

    \foreach \i in {-10,...,10} {
        \foreach \j in {-10,...,10} {
            \pgfmathtruncatemacro{\cond}{abs(\i)+abs(\j)==\r }
            \ifnum\cond=1
                \filldraw[Brown] (\i*\step,\j*\step) circle (1.7 pt);
            \fi
        }
    }

    \filldraw[Gray] (6.5*0.25, 0) circle (0.75pt);
    \filldraw[Gray] (7.5*0.25, 0) circle (0.75pt);
    \filldraw[Gray] (8.5*0.25, 0) circle (0.75pt);

    \begin{scope}[shift={(5.9,0.7)}]
        \node[
            font=\tiny,
            align=left, scale = 1
        ] at (-1.2,1) {
            \textcolor{Green}{$\bullet$} Goal\\
            \textcolor{Cerulean}{$\bullet$} Training $\leq 3$ steps\\
            \textcolor{Magenta}{$\bullet$} $4$ steps away\\
            \textcolor{Peach}{$\bullet$} $5$ steps away\\
            \textcolor{Brown}{$\bullet$} $n$ steps away
        };
    \end{scope}
\end{tikzpicture} 
    }
    \caption{
        Cross-scale generalization in the warm--cold lattice environment.
        \subref{fig:warm-cold}: 
        In the warm--cold lattice environment, the agent can navigate across an infinitely large lattice by moving up, down, left, or right to reach the rewarding goal location.
        The agent never observes its lattice position and is only given a warm (W) or cold (C) observation depending on it moving closer to or farther from the goal.
        To find the goal location in near-optimal time, it is sufficient for an agent to map past observation-action sequences into one of eight movement directions illustrated in the centre schematic. 
        Reasoning over infinitely many states is not required.
        \subref{fig:warm-cold-start-states}:
        Start-state distributions induce different training and test tasks.
        In the warm--cold lattice Partially Observable Markov Decision Process (POMDP) training start-states lie within three steps of the goal.
        Test start-states lie on an outside diamond and are $n$ steps away from the goal.
        The test start-state distribution samples uniformly from these states and is thus disjoint from the training start-state distribution.
        Solving the test tasks requires navigation and reasoning over many more time steps than during training---a form of OOD generalization.
    }
    \label{fig:motivation}
\end{figure}
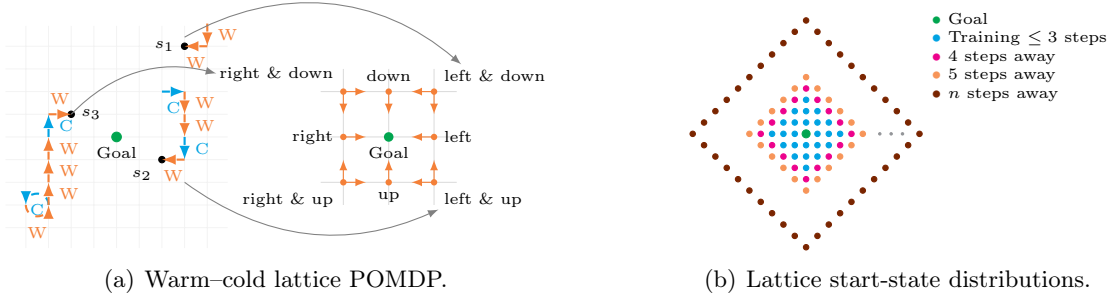

To illustrate how an intelligent agent may generalize to more complex tasks, consider the children's game of ``hot and cold'': 
Based on immediate feedback (warm or cold), a child can find an object in their surroundings.
The child does not need to understand their surroundings perfectly to complete the task, they only need to understand how to act based on a binary feedback signal. 
\Cref{fig:warm-cold} formalizes this task as a Partially Observable Markov Decision Process (POMDP) in which the agent does not observe its exact lattice position, but instead receives minimal binary feedback indicating whether it is moving closer to (warm feedback) or farther from (cold feedback) the goal. 
Under full observability, an agent would memorize optimal actions for each lattice position and would start to reason over a countably infinite set of lattice positions. 
Under limited feedback, the agent can construct an observation-action history.
Given this history, the same action becomes optimal in many distinct lattice positions; the agent need not distinguish among all underlying positions to find the goal in a near-optimal manner. 

\Cref{fig:warm-cold} illustrates how this task can be compressed into a much smaller abstract structure that models eight movement directions. 
If an agent learns this abstract structure, it can solve the task in near-optimal time.
Ultimately, this model enables generalization across task scales because it is independent of any lattice position or the distance between the start and goal, thereby achieving a form of OOD generalization.

A key element of our approach is the use of start-state distributions that induce different RL task instances within the same infinite-state POMDP.
\Cref{fig:warm-cold-start-states} plots different start-state distributions for the warm--cold lattice environment.
These start-state distributions determine how many steps an agent needs to navigate.
They defines the task's complexity level while assuming fixed transitions, rewards, and observations.

In this article we formalize this concept and identify conditions under which such OOD generalization is possible.
Specifically, we assume that an RL agent uses an abstraction~\citep{ravindran2004approximate,li2006towards,ravindran2004approximate} that determines which interaction histories can be treated as equivalent and which must be distinguished.
For POMDPs, we define a \emph{successor-weighted model reduction}, a model reduction~\citep{givan2003EquivalenceNotionsModel} variant that enables compression into much smaller abstract spaces than prior work~\citep{abel2018StateAbstractionsLifelong,lehnert2020SuccessorFeaturesCombine} would allow.
Our main result (\cref{thm:ood-pomdp}) is a bound on a RL agent's performance under a test start-state distribution that can be disjoint from a training start-state distribution the agent previously learned from.
Using this bound, we can identify conditions under which a successor-weighted model reduction achieves OOD generalization.
Lastly, we present a corollary and examples which demonstrate how OOD test performance can be improved by reducing the number of abstract states despite training and test tasks being probabilistically different.
This form of OOD generalization also encompasses generalization to more complex tasks that require reasoning over longer horizons or larger state spaces.

To unpack how state abstractions enable OOD generalization across tasks of different complexity levels,~\cref{sec:background} first covers preliminaries and outlines our proof strategy.
In~\cref{sec:infty-state-extension} extends existing proof techniques for performance loss bounds from finite to countably infinite state spaces.
This extension allows us to state a simulation lemma~\citep{kearns2002NearOptimalReinforcementLearning} and the telescoping performance difference theorem~\citep[Theorem 2]{xie2020q} for infinite state POMDPs.
Subsequently, we apply the existing state abstraction framework~\citep{ravindran2004approximate,li2006towards,abel2016OptimalBehaviorApproximate}, to infinite state POMDPs and formally define successor-weighted model reductions (\cref{sec:successor-weighted-model-reduction}).
Using these definitions and our proof technique extension, we derive our main result in~\cref{sec:ood-pomdp}.
\Cref{thm:ood-pomdp} presents a bound on the test performance loss in our OOD setting, thereby outlining conditions under which OOD generalization is possible.
Subsequently,~\cref{sec:estimation-error-types} presents different error types that limit an RL agent's test performance in our OOD setting and finds that reducing an agent's abstract state space size improves performance in test tasks that probabilistically differ from a previously used training task.
\Cref{sec:discussion} concludes with a discussion of how cross-scale generalization can be achieved by constraining an RL agent to work with a small number of abstract states and how our results motivate research into learning RL architectures from trial-and-error interactions that support this form of OOD generalization. 

\section{Background}
\label{sec:background}

We consider finite observation-action POMDPs such as the warm--cold task (\cref{fig:motivation}).
In the warm--cold task, the infinite latent state space $\statespace$ consists of the agent's true lattice position and past action. 
The agent selects actions from the finite action set $\actionspace$ to cause a state transition described by the transition function $p$ and then receives an observation $o \in \obsspace$ and a reward specified by the reward function $r$.
In contrast, a Markov Decision Process (MDP) assumes full observability of the underlying state $s$ and thus does not use an observation function $\eta$ or observation set $\obsspace$. 
Formal definitions are presented in~\cref{app:mdps-pomdps}.

\begin{figure}
    \centering

\begin{tikzpicture}[
  scale=1,
  transform shape,
  >=Stealth,
  node distance=0.5cm
]


\node[minimum width=1cm] (pomdp) at (0,0)  {$P$};

\node[minimum width=1cm,above left=of pomdp] (hist1)   {$M^\histories_\train$};
\node[minimum width=1cm,above right=of pomdp] (hist2)   {$M^\histories_\test$};
\node[minimum width=1cm,above=of hist1]       (natmdp1) {$M^\natnum_\train$};
\node[minimum width=1cm,above=of hist2]       (natmdp2) {$M^\natnum_\test$};
\node[minimum width=1cm,above=of natmdp1]     (absmdp1) {$M^\phi_\train$};
\node[minimum width=1cm,above=of natmdp2]     (absmdp2) {$M^\phi_\test$};

\draw[->]  (pomdp) -- node[below,xshift=-0.2cm] {$\pmb{d}_\start^\train$} (hist1);
\draw[->]  (pomdp) -- node[below,xshift=0.6cm] {$\pmb{d}_\start^\test$}  (hist2);


\draw[<->] (hist1) -- node[right]{\tiny bisimilar}(natmdp1);
\draw[<->] (hist2) -- node[right]{\tiny bisimilar}(natmdp2);
\draw[->]  (natmdp1) -- node[right]{\tiny $\phi$}(absmdp1);
\draw[->]  (natmdp2) -- node[right]{\tiny $\phi$}(absmdp2);
\draw[<->, bend left=40] (absmdp1) to node[above]{\tiny $\eps_r^\est$, $\eps_p^\est$} (absmdp2);

\end{tikzpicture}





    \caption{
        Out-of-distribution generalization via changing start-state distributions.
        Given two start-state distributions $\pmb{d}_\start^\train$ and $\pmb{d}_\start^\test$, a POMDP $P$ can be mapped to the history Markov Decision Processes (MDPs) $M^\histories_\train$ and $M^\histories_\test$ with state spaces that consist of the observation-action sequences that are observed since starting in a start state. 
        These history MDPs generate the same reward sequences with equal probability given the same action sequence and start conditions as the natural number MDPs $M^\natnum_\train$ and $M^\natnum_\test$---they are bisimilar~\citep{givan2003EquivalenceNotionsModel}.
        We generalize the existing state abstraction framework to construct the abstract MDPs $M^\phi_\train$ and $M^\phi_\test$.
    }
    \label{fig:proof-schematic}
\end{figure}
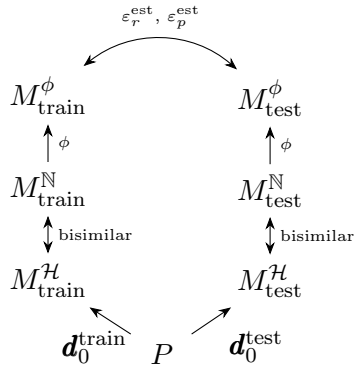

\Cref{fig:proof-schematic} outlines how we develop our results for infinite state POMDPs:
Given a start-state distribution, a POMDP can be mapped to an MDP whose state space consists of all possible observation-action histories an agent could observe. 
Such an observation-action history $h_t = (o_0,a_1,o_1,...,a_t,o_t)$ forms a Markovian state.
For finite observation and action spaces, the set of all possible observation-action histories $\histories$ is countably infinite\footnote{One can always pick an arbitrary ordering of the observation and action set and generate histories lexicographically. Thus the history set $\histories$ is countably infinite.}, and thus a history MDP can be viewed as equivalent to an MDP whose states are indexed by the natural numbers.
Therefore, it suffices to work with such natural number MDPs.
We assume that an RL agent uses a state abstraction function~\citep{ravindran2004approximate,li2006towards} that then compresses this history or natural number MDP into an abstract finite state MDP.
Such an RL agent then computes all predictions and decisions given the abstract state.

A task instance is characterized by a distribution over initial states.
We denote this start-state distribution as an infinite-dimensional probability vector $\pmb{d}_0$ (a function $\natnum \to \realnum$) that lies in the infinite-dimensional probability simplex $\probsimplex{\infty}$ (the set of vectors whose non-negative entries sums to one)\footnote{The set of $n$-dimensional probability vectors is denoted with $\probsimplex{n}$.}.
Picking different start-state distributions may induce probabilistically different history MDPs, requiring an agent to reason over a much larger state space or much longer horizon during testing than during training.
In this setting it is also possible that the agent cannot visit the same states during both testing and training.
Consequently, our training and testing settings encompass a form of OOD generalization.

In a POMDP, a policy $\pi$ maps an observation-action history to a probability distribution over actions. 
Such a policy is evaluated under a start-state distribution $\pmb{d}_0$ by its expected $\gamma$-discounted return (for $\gamma \in [0,1)$) denoted with $J$ and
\begin{equation}
    J(\pi, \pmb{d}_\start) \defeq \expec{\pi, \pmb{d}_\start}{ \sum_{t=1}^\infty \gamma^{t-1} r_t } . \label{eq:discounted-return-pi}
\end{equation}
This definition applies equally to the original POMDP and to the resulting history MDP: 
For a fixed policy and start-state distribution, both POMDP and MDP induce the same distribution over trajectories and reward sequences, resulting in the same expected discounted return $J$. 

\section{Results}
\label{sec:results}

To prove our main result in~\cref{sec:ood-pomdp}, we first extend existing results on deriving performance loss bounds from finite to countably infinite state spaces in~\cref{sec:infty-state-extension}.
Because we are considering finite observation and action spaces, we first note that the set of observation-action histories $\histories$ can be sorted lexicographically under an arbitrary ordering of the observation and action set.
This allows us to index history states using the natural numbers and consider MDPs with the natural numbers as a state space without loss of generality. 

\begin{proposition}[note={History--Natural Number MDP Equivalence},store={prop:hist-nat-num-mdp}]
    \label{prop:hist-nat-num-mdp}
    For every history MDP $M^\histories$ a natural number MDP $M^\natnum$ can be constructed using a bijection $f$ between the history set $\histories$ and the natural numbers $\natnum$.
    Moreover, the MDPs $M^\histories$ and $M^\natnum$ are bisimilar.
\end{proposition}

Bisimilarity between the history MDP $M^\histories$ and natural number MDP $M^\natnum$ ensures that the two MDPs are behaviourally equivalent. 
Therefore, we conduct our analysis on natural number MDPs and transfer our results back to history-based representations and then to infinite state POMDPs.
A proof of~\cref{prop:hist-nat-num-mdp} is listed in \cref{app:hist-nat-num-mdp-proof}

\Cref{sec:successor-weighted-model-reduction} defines a successor-weighted model-reduction abstraction, providing the definitions needed to derive our main results presented in~\cref{sec:ood-pomdp,sec:estimation-error-types}. 
\Cref{sec:infty-state-extension} presents how existing performance loss bound derivations can be extended to infinite-state POMDPs, which serve as the foundation for these results.

\subsection{Extending performance bounds from finite to countably infinite spaces}
\label{sec:infty-state-extension}

To derive a value loss bound for natural number MDPs, we define distribution-based norms for vectors and matrices.
These norms measure differences in state values with respect to visitation frequencies over the state or state-action space.
Given an infinite-dimensional probability vector\footnote{We denote the infinite dimensional probability simplex with $\probsimplex{\infty}$.} $\pmb{d} \in \Delta(\infty)$, we define the distribution weighted $L_1$ norm of an infinite-dimensional vector\footnote{An infinite-dimensional vector $\pmb{v} \in \realnum^\infty$ is equivalent to a function $\pmb{v}: \natnum \to \realnum$. An infinite-dimensional matrix $\pmb{P} \in \realnum^{\infty \times \infty}$ is equivalent to a function $\pmb{P}: \natnum \times \natnum \to \realnum$. Please also refer to~\cref{app:norms-and-vectors}.} $\pmb{v}$ as 
\begin{equation*}
    \norm{ \pmb{v} }_{\pmb{d}} = \sum_{i=1}^\infty \pmb{d}(i) |\pmb{v}(i)|.
\end{equation*}
Corresponding definitions for matrices can be found in~\cref{app:norms-and-vectors}.

The use of distribution-weighted norms differs from the more common use of maximum norms ($L_\infty$) for deriving simulation lemmas~\citep{kearns2002NearOptimalReinforcementLearning,jiang2020NotesTabularMethods,lobel2024optimal} that present a worst-case analysis.
In contrast, we measure value differences relative to a probability or visitation frequency vector over the state space, and thus quantify errors by the frequency with which states are visited. 
\Cref{app:norms-and-vectors} for formal definitions for infinite-dimensional vectors and matrices and their properties.

In this work, we consider $\gamma$-discounted visitation frequencies, similar to the Successor Representation (SR)~\citep{dayan1993ImprovingGeneralizationTemporal}.
Formally, we define a \emph{normalized SR} that conditions on a start-state distribution and a policy $\pi$ for countably infinite state spaces as
\begin{equation}
    \forall_{i \in \natnum}, ~ \pmb{d}^\pi_{\texttt{S}}(i) = \expec{ \pi, \pmb{d}_0 }{ ( 1 - \gamma) \sum_{t=1}^\infty \gamma^{t-1} \ind{ s_t=i } } \label{eq:sr-state-def}
\end{equation}
The SR definition in~\cref{eq:sr-state-def} differs from the usual SR definition~\citep{dayan1993ImprovingGeneralizationTemporal} in that the expectation does not condition on a specific start state and that the vector is scaled by a factor of $1 - \gamma$.
Therefore, we refer to this SR as normalized because $\pmb{d}^\pi_{\texttt{S}} \in \probsimplex{\infty}$.
The entry $\pmb{d}^\pi_{\texttt{S}}(i)$ predicts the frequency with which state $i$ is visited when using the start-state distribution $\pmb{d}_0$ and following policy $\pi$.

The definition in~\cref{eq:sr-state-def} can be extended to predict the visitation frequency across state-actions pairs.
For an indexing bijection $k: \natnum \times \actionspace \to \natnum$, we define a \emph{normalized state-action SR} as
\begin{equation}
    \forall_{(i,a) \in \natnum \times \actionspace}, ~ \pmb{d}^\pi(k(i,a)) = \expec{ \pi, \pmb{d}_0 }{ (1 - \gamma) \sum_{t=1}^\infty \gamma^{t-1} \ind{ s_t=i,a_t=a } }. \label{eq:sr-state-action-def}
\end{equation}
Here, the entry $\pmb{d}^\pi(k(i,a))$ predicts the frequency with which the state-action pair $(i,a)$ is visited when using start-state distribution $\pmb{d}_0$ and following policy $\pi$.

Using this framework, we can prove a simulation lemma for countably infinite state spaces and POMDPs\footnote{The maximum achievable reward is denoted with $\rmax$ and the maximum achievable state value with $\vmax = \rmax / (1 - \gamma)$.}.
Instead of providing a worst case bound for all states, value differences are measured with respect to a start-state distribution $\pmb{d}_0$.

\begin{proposition}[note={Simulation lemma for countably infinite state spaces},store={prop:simulation-lemma}]
    \label{prop:simulation-lemma}
    Consider two natural number MDPs $M$ and $\widehat{M}$ with a normalized SR $\widehat{\pmb{d}}^\pi_{\texttt{S}}$ for some policy $\pi$.
    If
    \begin{eqnarray}
        \forall_{a \in \actionspace} & \norm{ \pmb{r}^a - \widehat{\pmb{r}}^a }_{ \widehat{\pmb{d}}^\pi_{\texttt{S}} } \le \eps_r, & \norm{ \pmb{P}^a - \widehat{\pmb{P}}^a }_{ \widehat{\pmb{d}}^\pi_{\texttt{S}} } \le \eps_p, \label{eq:simulation-lemma-eps}
    \end{eqnarray}
    then,
    \begin{equation*}
        \norm{ \pmb{v}^\pi - \widehat{\pmb{v}}^\pi }_{ \pmb{d}_0 } \le \frac{1}{1 - \gamma} \left( \eps_r + \frac{\gamma}{2} \vmax \eps_p \right).
    \end{equation*}
\end{proposition}

Because~\cref{prop:simulation-lemma} uses a distribution-weighted norm and evaluates errors for countably infinite state spaces using a normalized SR $\widehat{\pmb{d}}^\pi_{\texttt{S}}$, a dependency on a specific policy $\pi$ is introduced that is not present in the original Simulation Lemma~\citep{kearns2002NearOptimalReinforcementLearning}.
Nevertheless, this approach permits high differences between rewards and transitions for less frequently visited states.

\subsubsection{Performance bounds for infinite state POMDPs}

Building on these results, we can derive the first bounds on an RL agent's test performance loss in infinite state POMDPs.
First we make the following assumption.

\begin{assumption}[Learning in finite steps]
    \label{asmpt:learning-finite-steps}
    An RL agent learns in a POMDP $P$ with start state distribution $\pmb{d}_0$ for $T$ time steps a policy $\widehat{\pi}$, and observes the set of histories $\histories_\obs \subset \histories$. 
    We denote the set of unvisited or unknown histories with $\histories_\unk = \histories \setminus \histories_\obs$.
    We assume that for the observed histories $\histories_\obs$ the learned policy $\widehat{\pi}$ is $\eps$-optimal and 
    \begin{equation*}
        \forall h \in \histories,~ \widehat{\pi}(h) = \begin{cases}
            \pmb{\pi}^\eps_h &\text{for}~h \in \histories_\obs, \\ 
            \pmb{\pi}_h &\text{for}~h \in \histories_\unk,
        \end{cases}
    \end{equation*}
    where $\pmb{\pi}^\eps_h,\pmb{\pi}_h \in \probsimplex{\abs{\actionspace}}$. 
    The action-selection probability vectors for a history $h$ of the optimal policy and an arbitrary policy are denoted with $\pmb{\pi}^\star_h$ and $\pmb{\pi}_h$ respectively.
    We assume that $\norm{ \pmb{\pi}^\eps_h - \pmb{\pi}^\star_h }_1 \le \eps$.
\end{assumption}

\Cref{asmpt:learning-finite-steps} does not make any assumption about how the policy $\widehat{\pi}$ is learned---it only assumes that the learned policy is near optimal for the observed histories.
Using~\cref{lemm:mrp-value-loss}, we can bound the value loss incurred by using the learned policy $\widehat{\pi}$ under the same start-state distribution $\pmb{d}_0$.
Formally, the value loss is the difference between the discounted return $J$ (defined in~\cref{eq:discounted-return-pi}) of the optimal policy $\pi^\star$ and the learned policy $\widehat{\pi}$.

\begin{proposition}[
    store={prop:value-loss-learning-finite-steps}
]
    \label{prop:value-loss-learning-finite-steps}
    Given a start-state distribution $\pmb{d}_0$, consider the normalized SRs $\pmb{d}^{\pi^\star}_{\texttt{S}}$ and $\pmb{d}^{\widehat{\pi}}_{\texttt{S}}$ for policies $\pi^\star$ and $\widehat{\pi}$.
    Under~\cref{asmpt:learning-finite-steps}, for either normalized SR $\pmb{d} \in \{ \pmb{d}^{\pi^\star}_{\texttt{S}}, \pmb{d}^{\widehat{\pi}}_{\texttt{S}} \}$, 
    \begin{equation*}
        J(\pi^\star,\pmb{d}_0) - J(\widehat{\pi},\pmb{d}_0) \le \frac{1}{1 - \gamma} \left( \pmb{d}(\histories_\obs) \eps \left( \rmax + \frac{\gamma}{2} \vmax \right) + \pmb{d}(\histories_\unk) \vmax \right) ,
    \end{equation*}
     where $\pmb{d}(\histories_\obs) = \sum_{i \in \histories_\obs} \pmb{d}(i)$ and $\pmb{d}(\histories_\unk) = \sum_{i \in \histories_\unk} \pmb{d}(i)$.
\end{proposition}

\Cref{app:value-loss-learning-finite-steps-proof} lists a formal proof of~\cref{prop:value-loss-learning-finite-steps}.
This proposition demonstrates that the loss bound depends predominantly on the probability $\pmb{d}(\histories_\unk)$ for small $\eps$.
If the probability $\pmb{d}(\histories_\obs)$ is close to one and the policy $\pi^\star$ or $\widehat{\pi}$ predominantly visits known states, then the value loss is small.
However, if the agent has not learned a near-optimal policy for many unknown states that are visited often under $\pi^\star$ or $\widehat{\pi}$, then the value loss increases.

\Cref{prop:value-loss-learning-finite-steps} highlights the need to use an abstraction function $\phi: \histories \to \statespace^\phi$ to generalize a learned policy to the unobserved histories $\histories_\unk$.
More importantly, only if the abstract state space $\statespace^\phi$ is finite, can the agent learn a policy that is near-optimal at all possible histories $h \in \histories$ (because the history set $\histories$ is countably infinite, but the agent can only visit a finite number of abstract states within finitely many time steps).

\subsubsection{Extending the telescoping performance difference theorem}

So far, our analysis expressed reward and transition functions as a set of vectors and matrices that are indexed by the action space (\cref{prop:simulation-lemma}).
To obtain the generalization bound in~\cref{thm:ood-pomdp}, we express reward and transition functions on state-action spaces.

Consider an infinite-dimensional vector $\pmb{f} \in \realnum^\infty$ that is indexed by state-action pairs using the indexing bijection $k : \natnum \times \actionspace \to \natnum$ that maps each natural number state–action pair to a unique index in the natural numbers.
For infinite-dimensional vectors $\pmb{f} \in \realnum^\infty$, the infinite Bellman operator $\bellmanopinf$ is defined as\footnote{Here, the vector $\nu(\pmb{f}, \pi) \in \realnum^\infty$ and $\pmb{P}_\natnum(k(s,a)) \nu_{\pmb{f}} = \sum_{i=1}^\infty \pmb{P}_\natnum(k(s,a))(i) \nu_{\pmb{f}}(i)$.}
\begin{equation}
    \bellmanopinf^\pi_{M^\natnum} \pmb{f}(k(s,a)) \defeq \pmb{r}_\natnum(k(s,a)) + \gamma \pmb{P}_\natnum(k(s,a)) \nu(\pmb{f}, \pi) ~\text{with}~\nu(\pmb{f}, \pi)(s') = \pmb{f}(s',\pi(s')) . \label{eq:bellmanopinf}
\end{equation}
This operator computes the expected one-step reward plus the discounted value of the next state $s'$ under policy $\pi$, similar to the standard Bellman operator~\citep{jiang2020NotesTabularMethods,howard1960DynamicProgrammingMarkov}, but for infinite-dimensional vectors.

To derive a bound on an agent's test performance loss with this notation, we re-derive the telescoping performance difference theorem~\citep[Theorem 2]{xie2020q} for countably infinite state spaces using the Bellman operator $\bellmanopinf$.
The telescoping performance difference theorem bounds the performance gap between an optimal policy and a policy $\pi_{\pmb{q}}$, that selects actions greedily with respect to some Q-value vector $\pmb{q}$, using the Bellman error $\pmb{q}-\bellmanop_M^{\pi_{ \pmb{q} }} \pmb{q}$.
In~\cref{app:telescoping-perf-diff-proof,lemm:policy-perf-diff} we prove that for any vector $\pmb{f} \in \realnum^\infty$ and start state distribution $\pmb{d}_0 \in \probsimplex{\infty}$, the policy performance difference 
\begin{equation}
    J(\pi, \pmb{d}_0) - J(\pi_{\pmb{f}}, \pmb{d}_0)  
    \le \frac{1}{1 - \gamma}
    \Big(
          \norm{ \pmb{f} - \bellmanopinf^{\pi_{\pmb{f}}}_{M^\natnum} \pmb{f} }_{\pmb{d}^\pi}
        + \norm{ \pmb{f} - \bellmanopinf^{\pi_{\pmb{f}}}_{M^\natnum} \pmb{f} }_{\pmb{d}^{\pi_{\pmb{f}}}}
    \Big) , \label{eq:telescoping-performance-difference-inf}
\end{equation}
where $\pi_{ \pmb{f} }$ is the policy that selects actions greedily with respect to the values stored in $\pmb{f}$.\footnote{Specifically, for a state $i \in \natnum$, $\pi_{ \pmb{f} }(i) \in \arg \max_{a \in \actionspace} \pmb{f}(k(i,a))$.}

\subsection{Successor-weighted model reductions}
\label{sec:successor-weighted-model-reduction}

Building on these results, we define a variant of a softened model reduction, called the successor-weighted model reduction.
This abstraction type enables us in~\cref{sec:ood-pomdp} to show how approximation and estimation errors contribute to an agent's test performance loss and how reducing the abstract state space size can improve performance in OOD test tasks.

A model reduction constructs an abstract MDP such that the abstracted MDPs rewards and transitions are approximate the original task as closely as possible.
One approach to formally defining model reductions is via projection operators~\citep{lehnert2018TransferModelFeatures}.
\Cref{app:projection-operators-and-model-reduction} presents up- and down-projection operators $\Phi^\uparrow$ and $\Phi^\downarrow$.
Intuitively, the term $\Phi^\downarrow \pmb{r}_\phi$ projects the abstract reward vector down to the state-action space of the orignal task.
The difference $\pmb{r} - \Phi^\downarrow \pmb{r}_\phi$ then computes the difference between the original task's reward vector and the down projected reward vector.
These differences are small if the state abstraction permits precise approximation of the reward function.

For transition matrices, approximation errors are computed differently because they are calculated for matrices.
Here, the state to state partition transition probabilities should match the abstract state to abstract state transition probabilities, as defined by~\citet{givan2003EquivalenceNotionsModel}.
\Cref{app:projection-operators-and-model-reduction} presents a more detailed description of how these projection operators are used to define approximation errors for model reductions.
We now formally define a successor-weighted model reduction.

\begin{definition}[Successor-Weighted Model Reduction for Countably Infinite States]
    \label{def:successor-weighted-model-reduction}
    For some natural number MDP $M^\natnum$ with start-state distribution $\pmb{d}_\start$ and some policy $\pi$, a successor-weighted model reduction under policy $\pi$ is a state abstraction function such that for $\eps_r^\phi,\eps_p^\phi > 0$,
    \begin{eqnarray*}
        \underbrace{
            \norm{ \pmb{r} - {\Phi^\downarrow} \pmb{r}_\phi }_{ \pmb{d}^\pi }
        }_{\substack{\text{reward}\\\text{approximation error}}} \le \eps_r^\phi, 
        & 
        \underbrace{
            \norm{ {\Phi^\uparrow} \pmb{P} - {\Phi^\downarrow} \pmb{P}_\phi }_{ \pmb{d}^\pi }
        }_{\substack{\text{transition}\\\text{approximation error}}} \le \eps_p^\phi,
    \end{eqnarray*}
    where $\pmb{d}^\pi$ is the normalized state-action SR under policy $\pi$ and start state distribution $\pmb{d}_\start$.
\end{definition}

The use of the normalized SR $\pmb{d}^\pi$ causes a successor-weighted model reduction to depend on a policy $\pi$.
Normally, a model reduction would not depend on any specific policy, because it is designed to reduce the MDP itself. 
Here, we make a softer assumption and bound the differences using the visitation distribution (the normalized SR) under a specific policy. 
As a result, larger differences in rarely visited states may be deemphasized by $\pmb{d}^\pi$, because they contribute to the norm with a smaller weight.
Apart from this policy dependency, this approach makes the abstraction assumption less restrictive than existing definitions~\citep{abel2016OptimalBehaviorApproximate} that rely on maximum norms and require errors to be bounded uniformly.

\subsubsection{Distribution-weighted model reductions enable higher compression}
\label{sec:high-compression-example}

\begin{figure}
    \centering
    \subfigure[
        Infinite chain example.    
    ]{
        \label{fig:infity-chain-abstraction}
\begin{tikzpicture}[scale=0.5]
    \begin{scope}[shift={(-4,1)}]
        \node[circle, Cerulean,  draw, thick] (S) at (-4.6,0.35) {};

        \node[circle, Green,  draw, thick] (S2) at ($(S)+(-4.9,0)$) {};   
        
        \path (S) -- coordinate[midway] (M) (S2);
        
        \draw[thin,gray,] (S) -- node[midway,below, gray, font = \tiny]{ left} (M);
        \draw[->,thin,gray] (M) -- node[midway,below, font = \tiny]{ $p$} (S2);
        \fill[gray] (M) circle (1.8pt);
        \draw[->,thin,gray]
        (M) to[out=70,in=110] node[above, font = \tiny]{ $1-p$} (S);
        \draw[thin,->, gray] (S) to[out=20,in=-20,looseness=6] node[right,font = \tiny] { right} (S);

    \end{scope}
    \begin{scope}[shift={(-4,-2)}]
        
        \draw[step=0.7cm,Green,very thin, opacity = 0.3] (-9.8,0) grid (-9.1,0.7);
        \draw[step=0.7cm,Gray,very thin, opacity= 0.3] (-9.1,0) grid (-7.7,0.7);
        \node[Green, font = \tiny] at (-9.5,0.35) {+1};
        \draw[Gray,very thin, opacity = 0.3] (-7.7,0) rectangle (-2.8,0.7);
        \node[gray] at (-6.6,0.35) {$\cdots$};
        
        \draw[step=0.7cm,Gray,very thin, opacity = 0.3] (-5.8,0) grid (-2.50,0.7);
        \node[gray] at (-2.2,0.35) {$\cdots$};
    
        \foreach \i [count=\j from 1] in {-9.8,-9.1, ...,-8.4}
        {
        \node[text=gray, font = \tiny] at (\i+0.35,-0.35) {  \j};
         }
    
        \node[text=gray, font = \tiny] at (-8.75,0.35) {  $\emptyset$};
        \node[text=gray, font = \tiny] at (-8.05,0.35) {  $\emptyset$};
        \foreach \i  in {-7.0, -5.2}{
        
            \node[text=gray] at (\i+0.35,-0.35){$\cdots$};
        }
        \node[text=gray, font = \tiny]  at (-3.2,0.35) { $\emptyset$};
        \node[text=gray, font = \tiny]  at (-4.6,0.35) { $\emptyset$};
        \node[text=gray, font = \tiny]  at (-5.2,0.35) { $\emptyset$};
        \node[text=gray, font = \tiny]  at (-3.8,-0.35) { $N$};
        
        \node[text=black, font = \tiny]  at (-3.8,0.35) { $\leftarrow$};
    \end{scope}

    \draw[->,Green  ,>=latex, very thin] (-13.50,-1.15) to (-13.5,1);
    \draw[->, Cerulean, ,>=latex, very thin] (-11.80,-1.15) to (-8.7,1);
    \draw[->, Cerulean, ,>=latex, very thin] (-12.70,-1.15) to (-9.0,1);

    \draw[->, Cerulean, ,>=latex, very thin] (-9.20,-1.15) to (-8.6,1);
    \draw[->, Cerulean, ,>=latex, very thin] (-8.40,-1.15) to (-8.4,1);

\end{tikzpicture}
    }
    \hspace{0.4in}
    \subfigure[
        Transition error plot.
    ]{  
        \label{fig:infty-chain-transition-error}
        \includegraphics[scale=1]{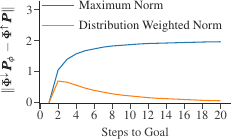}
    }
    \caption{
        Successor-weighted model reductions enable small abstract state spaces.
        \subref{fig:infity-chain-abstraction}: 
        The agent starts out at a state $N$ and then moves left towards the reward while observing first a left arrow ($\leftarrow$) and then empty observations ($\emptyset$) (until the goal state, the terminal state, is entered).
        This POMDP can be compressed into a two-state abstract MDP: 
        The terminal reward state (green) is modelled as one abstract state where every other state (and observation-action history) is mapped to the same abstract state.
        \subref{fig:infty-chain-transition-error}:
        Plot of the transition error for each $N$ for the example on the left.
        The distribution-weighted norm uses the normalized SR of the policy that always moves to the left.
        Because this norm is used, high differences in transitions for certain state-action pairs are de-emphasized.
        Consequently, even for a small $\eps_p^\phi$ setting, compression to a small two-state abstract MDP is possible.
    }
    \label{fig:chain-model-reduction}
\end{figure}
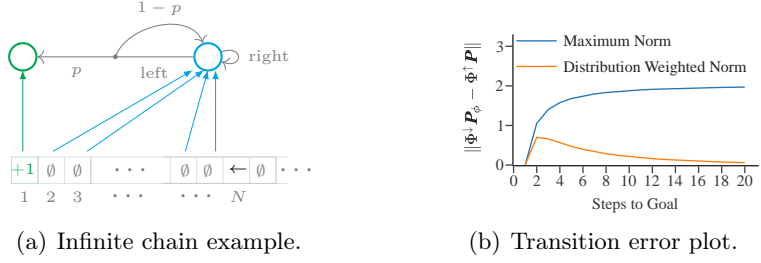

\citet{li2006towards} demonstrates that bisimulation relations induce partitions that are finer than those obtained from clustering Q-values, for example. 
Intuitively, bisimulation requires states to match both rewards and transition dynamics, while Q-value clustering abstractions only require equality of value functions.
\Cref{fig:chain-model-reduction} presents an example that demonstrates how a successor-weighted model reduction behaves differently and can construct a significantly smaller abstract MDP.
In this task, Q-values depend on the distance to the rewarding position, so states at different positions in the chain generally have different Q-values and cannot be merged by a Q-value clustering abstraction. 
A similar property holds for model reductions that are defined using the maximum norm where the clustering $\eps_p^\phi$ lies above one in almost all cases (\cref{fig:infty-chain-transition-error}).
In our setting, we consider a softened model reduction that measures reward and transition differences with respect to a state visitation distribution $\pmb{d}^\pi$. 
While this introduces a dependency of the abstraction to a specific policy and can induce non-determinism in the abstract MDP that is not present in the original task,~\cref{fig:chain-model-reduction} illustrates how this approach constructs much smaller abstract MDPs.

\paragraph{Policy dependency}
Because the choice of policy $\pi$ determines which state-action pairs are visited more often and have a higher weight in the normalized SR $\pmb{d}^\pi$, this policy determines how a successor-weighted model reduction is constructed.
For example, consider a policy that always moves to the right in the infinite chain example in~\cref{fig:infity-chain-abstraction}.
In this case, no weight would be assigned for position one or two and the agent would only observe a zero reward.
Consequently, a valid successor-weighted model reduction would be to collapse the entire task into a single state.
However, this collapse only happens due to the choice of the policy $\pi$.
It would not happen for a policy that always moves to the left, as illustrated in~\cref{fig:chain-model-reduction}.

\subsection{Out-of-distribution generalization in POMDPs}
\label{sec:ood-pomdp}

Building on these results, we now derive a bound for the performance gap between an optimal policy and an abstract policy, the policy that is optimal in the abstracted MDP.
To simplify notation, we define and define reward and transition approximation errors as
\begin{eqnarray}
    \pmb{\eps}_r^\phi \defeq \pmb{r} - {\Phi^\downarrow} \pmb{r}_\phi & \text{and} & \pmb{\eps}_p^\phi \defeq {\Phi^\uparrow} \pmb{P} - {\Phi^\downarrow} \pmb{P}_\phi. \label{eq:approximation-difference}
\end{eqnarray}
The following lemma presents a bound of an agent's performance loss that is incurred due to using a successor-weighted model reduction.

\begin{lemma}[note={Performance Loss under Successor-Weighted Model Reductions},store={lemm:performance-loss-succ-weighted-model-reduction}]
    \label{lemm:performance-loss-succ-weighted-model-reduction}
    For a state abstraction $\phi: \natnum \to \statespace_\phi$ with $\abs{ \statespace^\phi } < \infty$, consider a natural-number MDP $M^\natnum$ with optimal policy $\pi^\star$ and an abstract MDP $M^\phi$ with optimal abstract policy $\pi^\phi$.
    Then, for any start state distribution $\pmb{d}_0 \in \probsimplex{\infty}$ the value loss 
    \begin{align*}
        J( \pi^\star, \pmb{d}_0 ) - J( \pi^\phi, \pmb{d}_0 ) 
        \le \frac{1}{1 - \gamma} \sum_{ \pmb{d} } \norm{ \pmb{\eps}_r^\phi }_{ \pmb{d} } + \frac{\gamma}{2} \norm{ \pmb{\eps}_p^\phi }_{ \pmb{d} } \vmax,
    \end{align*}
    where the summation ranges over the two vectors $\pmb{d}^{\pi^\star}$ and $\pmb{d}^{\pi^\phi}$, the normalized SR under the policies $\pi^\star$ and $\pi^\phi$ when using the start-state distribution $\pmb{d}_0^\test$.
\end{lemma}

\Cref{lemm:performance-loss-succ-weighted-model-reduction} shows that the performance loss of using a policy that is learned with a successor-weighted model reduction is bounded by the reward and transition conditions stated in~\cref{def:successor-weighted-model-reduction}.
A key element of this bound are the policies used to evaluate the transition and reward errors:
The abstraction function must be both an approximate successor-weighted model reduction for both the optimal policy $\pi^\star$ and the abstract policy $\pi^\phi$ that is optimal in the abstracted MDP.
To achieve near-optimal performance, the abstraction $\phi$ must be constructed such that the differences in rewards and transitions between the original and abstracted tasks are small when evaluated with respect to the normalized SR of these two policies. 

While~\cref{lemm:performance-loss-succ-weighted-model-reduction} is stated for countably infinite state spaces, it also applies to finite state MDPs.\footnote{One can always extend a finite state MDP with infinitely many states that are self-looping and are not connected to any other state.}
\Cref{lemm:performance-loss-succ-weighted-model-reduction} improves upon bounds presented in prior work~\citep{abel2016OptimalBehaviorApproximate} by avoiding a linear dependency on the number of states and removing a factor of $1 / (1 - \gamma)$.

Before stating our main result, we first assume that an RL agent is trained using a training start-state distribution $\pmb{d}_0^\train$ and then evaluated under a different test start-state distribution $\pmb{d}_0^\test$. 
These distributions induce the abstract MDPs $M^\phi_\train$ and $M^\phi_\test$ respectively.
We define the difference between the rewards and transitions as
\begin{eqnarray}
    \pmb{\eps}_r^\ood \defeq \pmb{r}_{M^\phi_\train} - \pmb{r}_{M^\phi_\test} & \text{and} & \pmb{\eps}_p^\ood \defeq \pmb{P}_{M^\phi_\train} - \pmb{P}_{M^\phi_\test}. \label{eq:geneneralization-difference}
\end{eqnarray}
The following theorem presents a bound on the test performance loss incurred when using an abstract policy optimal in $M^\phi_\train$ in the original POMDP under a test start-state distribution $\pmb{d}_\start^\test$ that is possibly disjoint from the training start-state distribution $\pmb{d}_\start^\train$.

\begin{theorem}[note={Out-of-Distribution Generalization Bound for POMDPs}, store={thm:ood-pomdp}]
    \label{thm:ood-pomdp}
    Consider a POMDP $P$ with training and test start distributions $\pmb{d}_\start^\train$ and $\pmb{d}_\start^\test$ and an abstraction function $\phi$.
    Let $\pi^\star$ be an optimal policy under the test start-state distribution and let $\pi^\phi$ be the abstract policy optimal in the abstract training MDP $M_\train^\phi$ for an abstraction function $\phi$.
    Then,
    \begin{align*}
        J(\pi^\star, \pmb{d}_\start^\test) - J(\pi^\phi, \pmb{d}_\start^\test) 
        \le \frac{1}{1 - \gamma} \sum_{ \pmb{d} }
            \underbrace{
                \vphantom{\sum_{ \pmb{d} }}
                \norm{ \pmb{\eps}_r^\phi }_{ \pmb{d} } + \frac{\gamma}{2} \norm{ \pmb{\eps}_p^\phi }_{ \pmb{d} } \vmax
            }_{\text{approximation error}}
            + 
            \underbrace{
                \vphantom{\sum_{ \pmb{d} }}
                \norm{ \pmb{\eps}_r^\ood }_{ {\Phi^\uparrow} \pmb{d} } + \frac{\gamma}{2} \norm{ \pmb{\eps}_p^\ood }_{ {\Phi^\uparrow} \pmb{d} } \vmax
            }_{\text{OOD estimation error}}
    \end{align*} 
    where the summation ranges over the two vectors $\pmb{d}^{\pi^\star}$ and $\pmb{d}^{\pi^\phi}$, the normalized SR under the policies $\pi^\star$ and $\pi^\phi$ when using the test start-state distribution $\pmb{d}_\start^\test$.
\end{theorem}

\Cref{thm:ood-pomdp} decomposes the test performance loss into an approximation and an OOD estimation error.
In comparison, bias-complexity tradeoff in supervised learning~\citep[Ch. 5.2]{shalev-shwartz2014UnderstandingMachineLearning} and empirical risk minimization (ERM)~\citep{vapnik1991PrinciplesRiskMinimization} decomposes the test performance loss into an approximation and estimation error---similar to~\cref{thm:ood-pomdp}.
In ERM, the test performance loss arises because the model is trained on one dataset and subsequently evaluated on a different dataset, leading to an estimation error because the true underlying model cannot be estimated with perfect precision given a finite training dataset.
Both datasets are sampled from the same probabilistic process.

In this article, we consider reusing or generalizing a learned policy across probabilistically different tasks.
Even if an RL agent had access to infinite amounts of data and computation, it would, in the limit, only learn how to solve the training task.
Depending on the choice of start-state distributions and POMDP, high performance on the test task may not be achievable, because the test task is probabilistically too different from the training task.
For example, an agent that learns how to perform long division may not be able to reuse this knowledge to solve complex navigation tasks.
In this case, what the agent has learned during training is not relevant to the test task at hand and thus the resulting \emph{OOD estimation error} cannot be overcome---regardless of the used learning algorithm.
Ultimately, this article presents a model under which a form of OOD generalization can be achieved.
\Cref{app:ood-pomdp-thm} presents the formal proofs of~\cref{thm:ood-pomdp,lemm:performance-loss-succ-weighted-model-reduction}.

\subsection{Small abstract states spaces improve out-of-distribution generalization}
\label{sec:estimation-error-types}

The approximation error term in~\cref{thm:ood-pomdp} can always be reduced through the introduction of additional abstract states~\citep{abel2016OptimalBehaviorApproximate}, although this may in turn increase the OOD estimation error term. 
In total, our OOD generalization setting gives rise to four distinct sources of error that may cause an agent to select suboptimal actions during testing:
\begin{itemize}
    \item \textbf{Type 1 error:} The state abstraction aggregates histories into abstract states such that the agent is prevented from representing the optimal policy (approximation error).
    \item \textbf{Type 2 error:} Given a finite number of training steps, the agent cannot learn enough about all abstract states (resulting in a estimation error during training).
    \item \textbf{Type 3 error:} The test start-state distribution causes the agent to visit abstract states that were either not visited or infrequently visited during training (resulting in a high  estimation error during testing).
    \item \textbf{Type 4 error:} The training and test abstract MDPs are probabilistically different causing an abstract policy optimal in the abstract training MDP to select suboptimal actions during testing (resulting in a high  estimation error during testing).
\end{itemize}

The size of the abstract state space mediates a trade-off between Type 1 and Type 2 errors. 
Increasing the number of abstract states and constructing a finer state abstraction reduces approximation error (Type 1 error) at the cost of requiring the agent to learn in larger state spaces. 
When the number of transitions is fixed, learning accurate predictions for a greater number of abstract states becomes more challenging, as studied by sample complexity bounds presented in prior work (for example~\citet{jiang2020NotesTabularMethods}).

Depending on the specific training and test task instances, the size of the abstract state space may also mediate a trade-off between Type 1 and Type 3 errors. 
This tradeoff is characteristic to the presented OOD generalization setting.
It can occur even if Type 2 errors are zero and the agent learns a perfect approximation (or solution) of the training task.

To formalize this trade-off, we assume that an RL agent learns within $T$ time steps a $\eps$-optimal approximation of the training tasks abstract transition and reward functions for a subset of the abstract state space.

\begin{assumption}[Learning in finite steps with state abstractions]
    \label{asmpt:learning-finite-steps-abstraction}
    Given a fixed state abstraction $\phi: \histories \to \statespace^\phi$, an RL agent learns in a POMDP $P$ with training start-state distribution $\pmb{d}_0^\train$ for $T$ steps and estimates an abstract reward table $\widehat{\pmb{r}}_{M_\train^\phi}$ and transition table $\widehat{\pmb{P}}_{M_\train^\phi}$ for the abstract states $\statespace^\phi_\obs \subseteq \statespace^\phi$.
    For the other unknown states $\statespace^\phi_\unk = \statespace^\phi \setminus \statespace^\phi_\obs$ the agent makes arbitrary predictions.
    Thus, for every abstract state-action pair $(s_\phi,a)$,
    \begin{align*}
        \abs{ ( \widehat{\pmb{r}}_{M_\train^\phi} - \pmb{r}_{M_\train^\phi} ) (s_\phi,a) } 
            &= 
            \begin{cases}
                \eps &\text{if}~s_\phi \in \statespace^\phi_\obs \\ \rmax &\text{if}~s_\phi \in \statespace^\phi_\unk
            \end{cases} \\
        \norm{ ( \widehat{\pmb{P}}_{M_\train^\phi} - \pmb{P}_{M_\train^\phi} ) (s_\phi,a) }_1 
            &= 
            \begin{cases}
                \eps &\text{if}~s_\phi \in \statespace^\phi_\obs \\ 2 &\text{if}~s_\phi \in \statespace^\phi_\unk.
            \end{cases}
    \end{align*}
\end{assumption}

\Cref{asmpt:learning-finite-steps-abstraction} characterizes learning under a fixed training distribution where prediction quality depends on whether abstract states are observed during training. 
While assuming that the learned abstract reward and transition tables are $\eps$-optimal may seem restrictive, we note that existing sample complexity bounds (for example~\citet{jiang2020NotesTabularMethods}) provide worst-case bounds that hold with high probability.
These bounds determine the $\eps$ value stated in~\cref{asmpt:learning-finite-steps-abstraction}.

To model variability across OOD generalization experiments, we make an additional assumption on the abstract state visitation distributions that are encountered during in the training and test tasks. 

\begin{assumption}[Uniform abstract state distribution prior]
    \label{asmpt:uniform-prior-assumption}
    For any OOD generalization experiment, the abstract state visitation distributions $\pmb{d}^\train$ and $\pmb{d}^\test$ are assumed to be sampled from a uniform Dirichlet distribution.
\end{assumption}

Under these assumptions, we can derive the following corollary.

\begin{corollary}[note={Abstract state space size}, store={corl:state-space-dependency}]
    \label{corl:state-space-dependency}
    Under~\cref{asmpt:learning-finite-steps-abstraction,asmpt:uniform-prior-assumption},
    \begin{equation*}
        J(\pi^\star, \pmb{d}_\start^\test) - J(\pi^\phi, \pmb{d}_\start^\test) 
        \le O \left( \frac{1}{\abs{\statespace^\phi}^{\abs{\actionspace}}}
        + \left( 1 - \frac{\abs{ \statespace^\phi }-1}{T+\abs{ \statespace^\phi }-1} \right) \eps
        + \frac{\abs{ \statespace^\phi }-1}{T+\abs{ \statespace^\phi }-1} + B \right),
    \end{equation*}
    where the term $B$ upper bounds the Type 4 error.
\end{corollary}

\begin{figure}
    \centering
    \subfigure[
        $T=10$, $\eps=0.01$, $B=0$, $\abs{ \actionspace }=2$
    ]{
        \includegraphics[scale=1]{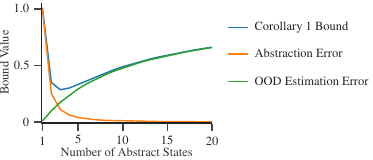}
    }\hspace{0.5in}
    \subfigure[
        $\eps=0.01$, $B=0$, $\abs{ \actionspace }=2$
    ]{
        \includegraphics[scale=1]{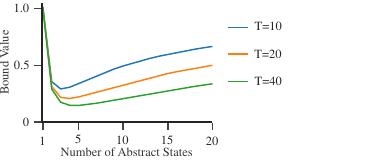}
    }
    \caption{
        The abstract state space size trades off test performance.
        Both plots visualize the bound presented in~\cref{corl:state-space-dependency}.
    }
    \label{fig:corollary-bounds}
\end{figure}

\Cref{corl:state-space-dependency} formalizes the trade-off between approximation and OOD estimation error as a function of abstract state space size. 
As the abstract state space size $|\statespace^\phi|$ increases, approximation errors decrease because there is less aggregation over histories, resulting in a more accurate approximation of the POMDP's rewards and transitions. 
When the number of training steps $T$ is fixed, larger abstract state spaces increase the likelihood of encountering unobserved or insufficiently estimated abstract states under an abstract state train–test visitation shift, which in turn raises Type 3 errors. 
Collectively, these factors result in a U-shaped relationship between OOD estimation error and approximation error (\cref{fig:corollary-bounds}).

\Cref{app:ood-estimation-error-bound} presents the proof for~\cref{corl:state-space-dependency}.
In the following section, we consider two examples to illustrate this trade-off between Type 1 errors (approximation errors) and Type 3 errors (errors due to abstract state visitation distribution shift).

\subsubsection{Warm--cold generalization experiment}
\label{sec:warm-cold-experiment}

To illustrate how the number of abstract states mediates the trade-off between Type 1 and Type 3 errors, we revisit the warm--cold task presented in~\cref{fig:motivation}. 
In this task, retaining the last few observation-action steps is sufficient for navigation to the rewarding goal location. 
Specifically, we consider a state abstraction mapping an observation history to a $k$-step suffix of the same sequence and
\begin{equation*}
    \phi: (o_0,a_1,o_1,...,a_T,o_T) \mapsto (a_{T-k},o_{T-k},...,a_T,o_T).
\end{equation*}
Here, the suffix $(a_{T-k},o_{T-k},...,a_T,o_T)$ forms an abstract state itself.
As the suffix length $k$ increases, the number of abstract states increases as well.

\begin{figure}
    \centering
    
    \subfigure[
        Combined number of mistakes.
    ]{
        \label{fig:U-curve-warm-cold-goal-dist-50}
        \includegraphics[scale=1]{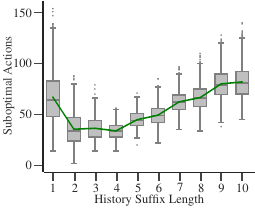}
    }
    \hfill
    \subfigure[
        Number of mistakes for steps where history suffix was found.
    ]{
        \label{fig:incorrect-prediction-warm-cold-goal-dist-50}
        \includegraphics[scale=1]{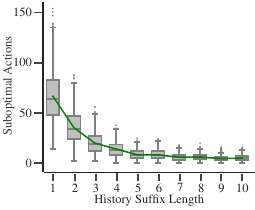}
    }
    \hfill
    \subfigure[
        Number of mistakes for steps where no history suffix was found.
    ]{
        \label{fig:table-misses-warm-cold-goal-dist-50}
        \includegraphics[scale=1]{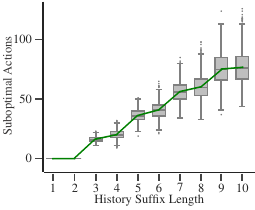}
    }

    \subfigure[
        Number of mistakes for different distances.
    ]{
        \label{fig:U-curve-warm-cold}
        \includegraphics[scale=1]{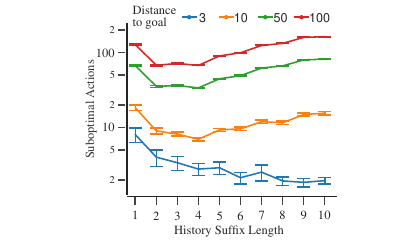}
    }
    \hfill
    \subfigure[
        Normalized number of mistakes for distances.
    ]{
        \label{fig:U-curve-warm-cold-scaled}
        \includegraphics[scale=1]{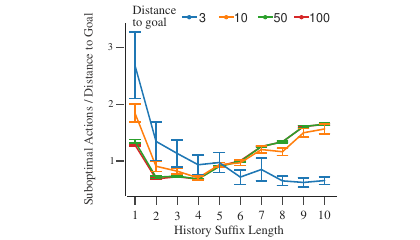}
    }
    
    \caption{%
        Reducing the number of abstract states improves OOD generalization.
        \subref{fig:U-curve-warm-cold-goal-dist-50}: 
        The number of mistakes (number of suboptimally selected actions) when starting the computed policy 50 steps away from the goal.
        For each history suffix length, each start state was sampled five times and the plot shows a standard box plot.
        The green curve shows the average number of suboptimally selected actions (average number of mistakes).
        This plot shows the sum of the values plotted in~\subref{fig:incorrect-prediction-warm-cold-goal-dist-50} and~\subref{fig:table-misses-warm-cold-goal-dist-50}. 
        \subref{fig:incorrect-prediction-warm-cold-goal-dist-50}:
        Number of mistakes when a history suffix was found in the policy dictionary.
        \subref{fig:table-misses-warm-cold-goal-dist-50}:
        Number of mistakes when a history suffix was not found in the policy dictionary.
        If no history suffix is found, then actions are selected uniformly at random. 
        \subref{fig:U-curve-warm-cold}:
        Number of suboptimal actions averaged across all repeats and all test start states that lie at a fixed distance to the goal.
        Error bars plot a 95\% confidence Standard Error of Mean (SEM).
        \subref{fig:U-curve-warm-cold-scaled}: 
        This plot shows the same data as \subref{fig:U-curve-warm-cold}, but the y-axis is scaled by the distance to the goal.
    }
    \label{fig:warm-cold-experiments}
\end{figure}

\Cref{fig:warm-cold-experiments} illustrates how test performance varies as a function of the history suffix length $k$ of this $k$-step suffix abstraction. 
For the training start-state distribution (\cref{fig:warm-cold-start-states}), an abstract policy is computed that approximates the optimal policy by enumerating all possible $k$-length suffix histories and all lattice positions that are reachable within a fixed horizon given a training start state. 
Using the underlying lattice position, we determine how often a specific action is optimal for a given $k$-step history suffix, and these optimal action counts are stored together with the corresponding suffix in a policy dictionary. 
When testing this policy dictionary in a test task, these counts are then used to compute an action-selection probability distribution and this action-selection probability distribution is then sampled to select the next action.
By enumerating all possible history suffixes that can occur from any training start state within a fixed horizon, this approach effectively computes a marginal over all (reachable) POMDP states and thus ensures that the resulting policy dictionary does not suffer from a Type 2 error.
Once this policy dictionary is computed, it is tested using different test start-state distributions.
\Cref{fig:warm-cold-start-states} illustrates the test start-state distributions for this experiment.
We measure test performance by counting the number of suboptimally selected actions.

\Cref{fig:U-curve-warm-cold-goal-dist-50} plots the number of suboptimal actions as a function of the history suffix length. 
Using a one- or two-step suffix leads to significantly poorer performance than using a four-step suffix. 
This is expected because longer histories encode more information about the agent's location and better inform which direction the goal lies in. 

Increasing the history suffix length beyond four steps results in diminished test performance. 
As the history length grows, the resulting abstract state more precisely captures the underlying lattice position (the underlying Markovian state of the POMDP) instead of only encoding information about a broad movement direction. 
Consequently, when the agent is tested on lattice positions not visited during training, the $k$-step history suffix obtained during testing may not be present in the previously computed policy dictionary. 
\Cref{fig:incorrect-prediction-warm-cold-goal-dist-50,fig:table-misses-warm-cold-goal-dist-50} show that these dictionary lookup failures account for the increased number of suboptimally selected actions. 
For higher history lengths, the frequency of suboptimal actions due to dictionary lookup failures is elevated (\cref{fig:table-misses-warm-cold-goal-dist-50}), whereas the number of suboptimal actions resulting from incorrect action predictions for suffixes found in the policy dictionary remains low (\cref{fig:incorrect-prediction-warm-cold-goal-dist-50}). 
These results indicate that not all history suffixes (abstract states) necessary for solving the test task are encountered under training start-state distribution, leading to Type 3 errors.
This example illustrates how the number of abstract states mediates a trade-off between Type 1 and Type 3 errors.

\Cref{fig:U-curve-warm-cold,fig:U-curve-warm-cold-scaled} illustrate how the test performance loss varies as the chosen test start-state distribution moves the agent's start location farther away from the goal state. 
In~\cref{fig:U-curve-warm-cold}, the blue curve shows that if the agent is tested in-distribution (the test start-state distribution is the same as the training start-state distribution), then test performance becomes almost optimal and the number of suboptimal actions drops down to two.
Note that in the warm--cold task the agent must first select at least two actions to position itself in the lattice to obtain information about which direction it should move into.
As the test start locations move further away from the goal, the number of suboptimal actions increases in absolute numbers, but not in relation to task scale (distance to goal), as illustrated in~\cref{fig:U-curve-warm-cold-scaled}.

This generalization experiment demonstrates that the size of the abstract state space plays a crucial role in achieving high OOD test performance. 
Too few abstract states may lead to a high approximation error (Type 1 error), whereas too many may shift the distribution of abstract state visits (Type 3 error). 
We note that these results are not influenced by a Type 2 error, because the policy dictionary is not obtained through trial-and-error learning.
Instead, it is obtained by enumerating all possible history suffixes.
\Cref{app:experiments} presents the details for this experiment.

\subsubsection{The sign chain example}
\label{sec:sign-chain}

\begin{figure}
    \centering
    \subfigure[Sign-chain task]
    {
\begin{tikzpicture}[scale=0.6]

    \begin{scope}[shift={(-4,0)}]
        \draw[step=0.7cm,gray,very thin, opacity= 0.3] (-0.3,0) grid (2.8,0.7);
        \draw[step=0.7cm,gray,very thin, opacity= 0.3] (4.18,0) grid (7.3,0.7);
        
        \node at (-0.8,0.35) {$\cdots$};
        \node at (7.8,0.35) {$\cdots$};
        
        \node[font = \tiny] at (1.1,0.35) {$\rightarrow$};
        
        \node[font = \tiny] at (1.15, -0.7) {A};
        \draw[->, thick] (1.15,-0.5) -- (1.15,0);
        
        \node[font = \tiny] at (6.0,0.35) {$\leftarrow$};
        
        \node[font = \tiny] at (6.0,-0.8) {B};
        \draw[->, thick] (6.0,-0.5) -- (6.0,0);
        
        \foreach \i in {0, 1.5, 2.1, 4.25, 5.0, 6.4}{
            \node[text=gray, font = \tiny] at (\i+0.25,0.35) {$\emptyset$};
        }
    \end{scope}
    \begin{scope}[shift={(-4,-2)}]
        
        \draw[step=0.7cm,gray,very thin, opacity= 0.3] (1.8,0) grid (2.8,0.7);
        \draw[step=0.7cm,gray,very thin, opacity= 0.3] (4.18,0) grid (5.18,0.7);

        \draw[step=0.7cm,gray,very thin, opacity= 0.3] (-1.7,0) grid (1.0,0.7);
        \draw[step=0.7cm,gray,very thin, opacity= 0.3] (5.98,0) grid (8.68,0.7);
        
        \node at (-1.7-0.5,0.35) {$\cdots$};
        \node at (8.68+0.5,0.35) {$\cdots$};

        \node at (1.45,0.35) {$\cdots$};
        \node at (5.18+0.45,0.35) {$\cdots$};

        \node[font = \tiny] at (-0.1-0.25,0.3) {$\rightarrow$};
        
        \node [font = \tiny] at (-0.1-0.25,-0.8) {C};
        \draw[->, thick] (-0.1-0.25,-0.5) -- (-0.1-0.25,0);
        
        \node[font = \tiny] at (7.8-0.45,0.3) {$\leftarrow$};
        
        \node[font = \tiny] at (7.8-0.45,-0.8) {D};
        \draw[->, thick] (7.8-0.45,-0.5) -- (7.8-0.45,0);
    
        \foreach \i in {-0.6,0.8,3,5,7.1,8.5} {
            \node[text=gray, font = \tiny] at (\i-0.45,0.35) {$\emptyset$};
        }

    \end{scope}
    
    \begin{scope}[shift={(-3,-1)}]
        \draw[step=0.7cm,Green,very thin] (2.09,0) grid (2.8,0.7);
        \node[text=Green, font = \tiny] at (2.4,0.3) { +1};
    \end{scope}

    \draw[->, black, bend left=20,>=latex] (0.1,-1.6) to (-0.2,-1.0);
    \draw[->, black, bend right=20,>=latex] (-1.2,-1.6) to (-0.8,-1.0);
    \draw[->, black, bend left=20,>=latex] (-1.2, 0.3) to (-0.8,-0.3);
    \draw[->, black, bend right=20,>=latex] (0.2, 0.3) to (-0.2,-0.3);
   
\end{tikzpicture}
        \label{fig:sign-chain-task}
    }\hfill
    \subfigure[Performance when starting in state C or D]{
        \label{fig:sign-chain-example-performance}
        \hspace{2mm}
        \includegraphics[scale=1]{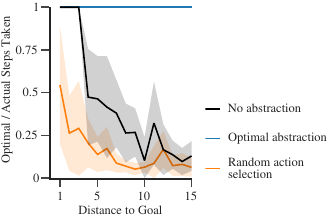}
        \hspace{2mm}
    }
    \caption{
        Sign-chain task example. 
        \subref{fig:sign-chain-task}: 
        During training, the agent either starts at location A or B and immediately observes a sign: either move left ($\leftarrow$) or move right ($\rightarrow$).
        At all other locations the agent only observes the same $\emptyset$ observation.
        All transitions are deterministic, and the rewarding state is terminal.
        During testing, the agent either starts at location C or D.
        These states are placed at a higher and variable distance to the rewarding goal state.
        \subref{fig:sign-chain-example-performance}: 
        The y-axis plots the optimal number of steps divided by the actual steps taken when using an optimal abstraction, using no abstraction, and using random action selection.
        The shaded area plot 95\% confidence SEM across five repeats.
    }
    \label{fig:sign-chain-example}
\end{figure}
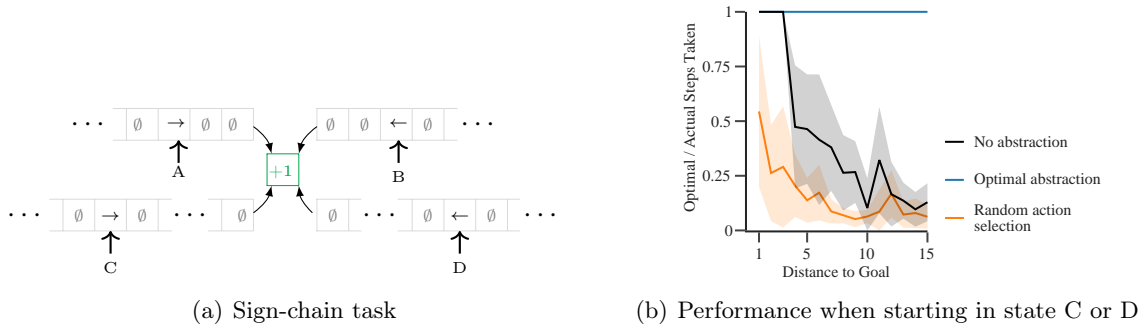

\Cref{fig:sign-chain-example} presents the sign chain task illustrating a scenario where Type 1 and Type 4 errors are zero for only using three abstract states. 
The training start-state distribution places the agent either at state A or state B.
Initially, the agent observes a sign indicating whether to move left or right to reach the goal. 
An optimal abstraction would construct three abstract states: one for when the move-left sign is observed, another for when the move-right sign is observed, and a third for the rewarding goal state at the centre. 
Under a training start-state distribution that samples states A and B equally often, it is possible to learn an optimal policy. 
This policy generalizes to starting in states C and D, regardless of the agent's distance from the goal, because the state abstraction processes the initial sign observation to place the agent in the correct abstract state. 

\Cref{fig:sign-chain-example-performance} compares the test performance when using the optimal abstraction and when using no abstraction.
Here, a policy dictionary maps either the initially observed sign or the entire observation-action history since start to an action-selection distribution.
Under the optimal abstraction, OOD generalization is achievable in this example.
However, when predicting optimal actions conditioned on the full observation history test performance degrades quickly to random action selection because five or ten step histories were not observed during training.

This example highlights how a state abstraction enables OOD generalization: 
The underlying POMDP states that are reachable from the training and test start-state distributions are completely disjoint.
Using the correct state abstraction enables the RL agent to solve the test task optimally.
Moreover, the choice of start-state distribution also determines if good test performance is possible. 
For example, if the agent were to only start in states A and C during training and would then be tested in states B and D, then experiment would fail because an RL agent would no have learned how to process a move left (``$\leftarrow$'') sign.

\section{Discussion}
\label{sec:discussion}

The Big World Hypothesis~\citep{javed2024thebigworld} states that in many decision-making problems, the environment is much larger and more complex than what the agent can observe. 
Hence it must rely on approximate solutions and simpler internal representations that are sufficient for achieving its goals. 
In our setting, this idea arises when we use a small abstract state space to solve a POMDPs with an arbirary, and infinite, state space. 
Here, the entire observation and action history of the agent forms a Markovian state.
Since there are infinitely many histories, an RL agent must compress this information. 
This motivates our focus on small finite abstract state spaces that support near-optimal decision-making in environments that are too large to model exactly.
By considering small finite abstract state spaces, we demonstrate how OOD generalization is feasible because an intelligent system tracks an approximate abstract state across different task scales.

In this context POMDPs provide a natural framework for studying OOD generalization. Instead of learning a policy, value function, or world model on some fixed state space, the agent must construct and track a mental abstract state. 
Instead of simply increasing model size or over-parameterizing an architecture, our results suggest that resource-constrained architectures may actually be necessary.
Our examples demonstrate that constraining a system to work with fewer abstract states necessary to generalize to more complex tasks. 

This concept is also related to how computation is modelled in a Turing machine, a system that processes a sequence of inputs from an infinite-length memory tape using a finite number of internal states to decide on the next output. 
Hence, an abstract state can be viewed as a form of working memory that is used for decision-making. 
From this perspective, generalization in partially observable environments depends on how well the internal state captures the underlying principles and rules of a task, enabling prediction of future decisions.

\citet{kirtland2026memory} show that memory functions can be interpreted as abstractions over trajectory histories, enabling existing abstraction theory~\citep{li2006towards,abel2016OptimalBehaviorApproximate} to be applied in POMDP settings.  
Their framework classifies memory abstractions as model-preserving, $Q^*$ preserving, or $\pi^*$ preserving and establishes relationships between different classes of memory functions.
Our approach differs from~\citeauthor{kirtland2026memory} in that we treat the start-state distribution explicitly. 
Given a POMDP and a start distribution, we construct a corresponding history MDP (called the trajectory MDP in their work) and analyze how abstraction behaves when this distribution shifts between training and evaluation.
Moreover, our analysis starts out with a generic assumption that a state abstraction is used and then identifies a new variant of a softened model reduction---the successor weighted model reduction.
While our framework focuses on finite observation-action POMDPs, we do not make any additional assumptions about fixed history or episode lengths~\citep{efroni2022ProvableReinforcementLearning,dann2021ValueFunctionGapsImproved} or the complexity or structure of the task~\citep{jin2021BellmanEluderDimension,du2019ProvablyEfficientRL}.
How our results relate to other memory construction methods for POMDPs~\citep{eberhard2025memorytraces,ni2022recurrentpomdps} or how to obtain tighter value loss bounds is left to future work.

Prior works on partially observable systems with large or infinite state space~\citep{ross2007bayespomdp,doshi2009infinitepomdp,grinberg2012average} mainly focus on learning or policy evaluation in POMDPs. 
In contrast, we present a model of how OOD generalization can be achieved in RL and assume that training and test start-state distributions are disjoint. 

In current practice, generative AI systems---in particular Large Language Models (LLMs)---are scaled by increasing model capacity and dataset size to reach a specific performance level~\citep{kaplan2020ScalingLawsNeural,belkin2019DoubleDescent}.
The big world hypothesis discussed earlier hypothesizes that increasing model capacity alone may not be sufficient to capture all relevant aspects of the environment.
As in real-world settings, intelligent agents frequently encounter environments that differ in scale or complexity from those observed during training. 
Therefore, the ability to generalize beyond the training distribution is fundamental to building robust and scalable representations that capture the principles or rules that underlie a task. 
Our results, therefore, suggest that robust generalization can be achieved by using fewer abstract states instead of increasing model capacity.

\subsection{Limitations and future directions}

In this article, we present a model of how an RL agent could achieve OOD generalization and conclude that constraining abstract state spaces to be small and finite improves test performance. 
The design of algorithms that learn and maintain these abstractions in a POMDP is left for future research.

Our model assumes that the agent employs a single abstraction function for both learning and generalization to a test task. 
Developing methods for composing multiple abstraction functions or abstract reward and transition models, analogous to the Options framework~\citep{sutton1999MDPsSemiMDPsFrameworka}, and for constructing a library of such abstractions would enable an RL agent to adapt more flexibly to novel tasks. 
Furthermore, this capability would reduce Type 4 errors and is essential for agents operating in a continual learning setting that requires ongoing adaptation. 
We leave the exploration of these topics to future work.

Moreover, one may also obtain more flexible adaptation by only reusing a portion of the learned model. 
For example, reuse of the abstraction function~\citep{lehnert2020RewardpredictiveRepresentationsGeneralize} or composing different transition and reward functions~\citep{liu2020hiearchicalclustering,franklinCompositionalClusteringTask2018} would enable more flexible generalization and allow an agent to further reduce Type 4 errors.

Lastly,~\cref{thm:ood-pomdp} considers model reductions. 
While prior work shows that model reductions are suitable for generalizing across tasks with different transitions and rewards, other state abstraction types might afford the OOD generalization similar to the ones studied in this article. 
While we are not aware whether OOD generalization with other abstraction types is possible, we leave such an investigation to future work.

\section{Conclusion}
\label{sec:conclusion}

Traditionally, RL algorithms such as Q-learning~\citep{watkins1992Qlearning}, DQN~\citep{mnih2015human}, or AlphaGo~\citep{silver2016MasteringGameGo} are designed to predict certain quantities, for example, the discounted return, as accurately as possible. 
This article presents an alternative perspective: 
By constraining an agent to operate over a very small number of abstract states, one may gain the ability to generalize to more complex, larger-scale tasks while giving up the ability to precisely predict certain quantities or outcomes. 
We hope that our results motivate further research on RL algorithms that learn architectures that generalize and scale to more complex tasks.

\section*{Acknowledgement}
We would like to thank Jonas Lehnert for helpful feedback and comments on earlier drafts of this paper.
Nasehatul Mustakim was funded in part by the Natural Sciences and Engineering Research Council of Canada.

\appendix
\crefalias{section}{appendix}
\crefalias{subsection}{appendix}
\crefalias{subsubsection}{appendix}

\section{Definitions, lemmas, and proofs}
\label{app:theory}

\subsection{MDPs and POMDPs}
\label{app:mdps-pomdps}

\begin{definition}[MDP]
    \label{def:mdp}
    An MDP is a quintuple $M = \langle \statespace, \actionspace, p, r, \gamma \rangle$ where
    \begin{enumerate}[nosep,topsep=0pt,leftmargin=2em]
        \item a state space $\statespace$,
        \item a finite action space $\actionspace$,
        \item a transition function specifying transition probabilities with $p(s' | s,a)$,
        \item an expected reward function $r: \statespace \times \actionspace \to \realnum$, and
        \item a discount factor $\gamma \in [0, 1)$.
    \end{enumerate}
\end{definition}

\begin{definition}[POMDP]
    \label{def:pomdp}
    A finite observation and action POMDP is a septuple $P = \langle \statespace, \actionspace, \obsspace, p, r, \eta, \gamma \rangle$ where
    \begin{enumerate}[nosep,topsep=0pt,leftmargin=2em]
        \item a state space $\statespace$ (infinite),
        \item a finite action space $\actionspace$,
        \item a finite observation space $\obsspace$,
        \item a transition function specifying transition probabilities with $p(s' | s,a)$,
        \item an expected reward function $r: \statespace \times \actionspace \to \realnum$, 
        \item an observation function specifying the probability of an observation with $\eta(o|s)$, and
        \item a discount factor $\gamma \in [0, 1)$.
    \end{enumerate}
\end{definition}

\begin{definition}[History MDP]
    \label{def:history-mdp}
    Given a POMDP $P = \langle \statespace, \actionspace, \obsspace, p, r, \eta, \gamma \rangle$ and a start state distribution $d_\start$ we can construct a history MDP $M^\histories = \langle \histories, \actionspace, p_\histories, r_\histories, \gamma \rangle$ such that
    \begin{enumerate}[nosep,topsep=0pt,leftmargin=2em]
        \item the state space is the set of histories $\histories = \{ (o_0,a_1,o_1,...,a_t,o_t) \}$.
        \item the finite action space is the set $\actionspace$,
        \item the transition function is defined as $$p_\histories(\underbrace{h_t a o_{t+1}}_{=h_{t+1}} | h_t,a) = \sum_{s_t,s_{t+1} \in \statespace} \prob{s_t|h_t} p(s_{t+1}|s_t,a) \eta(o_{t+1}|s_{t+1})$$
        \item the expected reward function is defined as $r_\histories(h_t,a) =\sum_{s_t \in \statespace} \prob{s_t|h_t} r(s_t,a),$
        \item and the discount factor $\gamma \in [0, 1)$.
    \end{enumerate}
\end{definition}

\begin{definition}[Natural Number Markov Reward Process (MRP)]
    \label{def:natnum-mrp}
    For an MDP $M^\natnum = \langle \natnum,\actionspace,p_\natnum,r_\natnum,\gamma \rangle$ and a policy $\pi : \natnum \to \probsimplex{\abs{\actionspace}}$, an MRP is a pair $\langle \pmb{P}^\pi, \pmb{r}^\pi \rangle$ of an infinite-dimensional transition matrix $\pmb{P}^\pi: \natnum \times \natnum \to [0,1]$ and an infinite-dimensional reward vector $\pmb{r}^\pi: \natnum \to [0,\rmax]$. 
    For the reward vector we have that
    \begin{equation}
        \forall s \in \natnum,~ \pmb{r}^\pi(s) = \sum_{a \in \actionspace} \pi(a|s) r(s,a)
    \end{equation}
    and for each row of the transition matrix we have that
    \begin{equation}
        \forall s,s' \in \natnum,~ \pmb{P}^\pi(s) \in \probsimplex{\infty} ~\text{and}~ \pmb{P}^\pi(s,s') = \sum_{a \in \actionspace} \pi(a|s) p(s'|s,a).
    \end{equation}
\end{definition}

\subsection{Bisimilarity between history and natural number MDPs}
\label{app:hist-nat-num-mdp-proof}

\noindent \textbf{Restatement of \cref{prop:hist-nat-num-mdp}.}
    \getkeytheorem[body]{prop:hist-nat-num-mdp}

\begin{proof}
    Since $\histories$ is countably infinite, there exists a bijection $f:\histories \to\mathbb{N}$ (e.g. by picking an arbitrary ordering of the observation space $\obsspace$ and action space $\actionspace$ and sorting all histories in lexicographic order).
    Then, for $n \in \natnum$, the natural number MDP $M^\natnum = \langle \natnum, \actionspace, p_\natnum, r_\natnum, \gamma \rangle$ becomes
    \begin{eqnarray*}
        r_\natnum(n,a) = r_\histories(f^{-1}(n),a)  & \text{and} & p_\natnum(n' | n,a) = p_\histories(f^{-1}(n') | f^{-1}(n),a).
    \end{eqnarray*}
    Since $f^{-1}:\natnum\to H$ is a bijection, and $h' = f^{-1}(n')$ ranges over all of $\histories$, we can write,
    \begin{equation}
        \sum_{n'\in\natnum} p_\natnum(n'| n,a) =
        \sum_{n'\in\natnum} p_\histories(f^{-1}(n') | f^{-1}(n),a)= \sum_{h'\in \histories} p_\histories(h' | f^{-1}(n),a)=1 .
    \end{equation}
    Hence, construction of $M^\natnum$ is valid. 
    Next, we show that these two MDPs are bisimilar.
    Define a relation $R \subseteq \histories \times \natnum$ where each history state $h \in \histories$ is associated with its corresponding natural number state $n= f(h) $.
    For any $(h,n)\in R$ and any action $a\in \actionspace$,
    \begin{equation}
        r_\natnum(n,a) = r_\histories(f^{-1}(n),a) = r_\histories(h,a).
    \end{equation}
    Also, for any $(h',n')\in R$,
    \begin{equation}
        p_\natnum(n' | n,a) = p_\histories(f^{-1}(n') | f^{-1}(n),a)=  p_\histories(h' | h,a).
    \end{equation}
    Therefore, the relation $R$ is a bisimilar relation and MDPs $M^\histories$ and $M_\natnum$ are bisimilar.
\end{proof}

\subsection{Vectors, matrices, and norms}
\label{app:norms-and-vectors}

In this article, an infinite-dimensional vector $\pmb{v}$ is a function $\pmb{v}: \natnum \to \realnum$ and we denote the space of these vectors (functions) with $\realnum^\infty$.
An infinite-dimensional matrix $\pmb{M}$ is a function $\pmb{M}: \natnum \times \natnum \to \realnum$ and we denote the space of these matrices (functions) with $\realnum^{\infty \times \infty}$.
For these infinite-dimensional vectors we define norms as follows.

\begin{definition}[$L_1$ and $L_\infty$ norms]
    \label{def:l1-linf-norms}
    We consider a state space of natural numbers, $\mathcal S = \mathbb N$. For any function $\pmb{v}:\mathbb N \to \mathbb R$, the $L_1$ and $L_\infty$ norms are defined as
    \begin{align}
        \|\pmb{v}\|_1 = \sum_{i=0}^{\infty} |\pmb{v}(i)| ~~~\text{and}~~~ \|\pmb{v}\|_\infty = \max_{i \in \mathbb N} |\pmb{v}(i)| .
    \end{align}
\end{definition}

\begin{definition}[Weighted $L_1$ norm] 
\label{def:weighted-L1-norm}
    We consider a state space of natural numbers (i.e., $\mathcal S=\mathbb N$). For any probability distribution $\pmb{p}$ over $\mathbb N$ and any function $\pmb{v}:\mathbb N \rightarrow \mathbb R$, we define
    \begin{equation}
        \|\pmb{v}\|_{\pmb{p}} = \sum_{i =0}^\infty \pmb{p}(i) |\pmb{v}(i)|.    
    \end{equation}
    Then the weighted matrix $L_1$ norm for any probability distribution
    $\pmb{p}$ over $\mathbb N$ and any row-stochastic matrix
    $\pmb{M}$ is defined as
    \begin{equation}
        \|\pmb{M}\|_{\pmb{p}}=\sum_{i=0}^\infty \pmb{p}(i)\sum_{j=0}^\infty |\pmb{M}(i,j)|=
    \sum_{i =0}^\infty \pmb{p}(i) \|\pmb{M}(i)\|_1 .\label{eq:weighted matrix l1 norm}
    \end{equation}
\end{definition}

For these definitions we prove the following lemmas and propositions.

\begin{proposition}[Soundness of Weighted $L_1$ Norm]
    \label{prop:soundness-weighted-l1-norm}
    The weighted $L_1$ norm $\| \cdot \|_{\pmb{p}}$ (\cref{def:weighted-L1-norm}) satisfies the defining properties of a norm, i.e., triangle inequality, absolute homogeneity, and positiveness. Hence, it is a norm.
\end{proposition}
\begin{proof}[\Cref{prop:soundness-weighted-l1-norm}]
    For a probability distribution $\pmb{p}$ over natural numbers $\mathbb N$, functions $\pmb{v},\pmb{w}:\mathbb N \rightarrow \mathbb R$,
    and a scalar $\alpha$,
    \begin{enumerate}
        \item Triangle inequality:
        \begin{equation}
            \|\pmb{v}+\pmb{w}\|_p = \sum_{i =0}^\infty \pmb{p}(i) | \pmb{v}(i)+\pmb{w}(i)| 
            \le \sum_{i =0}^\infty  \pmb{p}(i) | \pmb{v}(i)| +  \sum_{i =0}^\infty   \pmb{p}(i) | +\pmb{w}(i)| 
            = \| \pmb{v} \| _p + \| \pmb{w} \| _p 
        \end{equation}
        \item Absolute Homogeneity:
        \begin{equation}
            \| \alpha \pmb{v} \| _p = \sum_{i =0}^\infty   \pmb{p}(i) |\alpha  \pmb{v}(i)| =|\alpha |  \sum_{i =0}^\infty   \pmb{p}(i) | \pmb{v}(i)| = |\alpha |  \| \pmb{v} \| _p
        \end{equation}
        \item Positiveness:
        Since $\pmb{p}(i) >0$ and $| \pmb{v}(i)| \ge 0$,
        then $ \| \pmb{v} \| _p=  \sum_{i =0}^\infty   \pmb{p}(i) | \pmb{v}(i)|  \ge 0 $.
        Moreover,
        $ \| \pmb{v} \| _p=0  $ if and only if $ \pmb{v}(i) = 0 , \forall_{i \in \mathcal{S}} .  $  
    \end{enumerate}
\end{proof}

\begin{lemma}[Weighted $L_1$ Norm under Stochastic Matrix]
    \label{lemm:weighted-l1-norm-stochatic-matrix}
    Let $\pmb{p}$ be a probability distribution over $\mathcal S=\mathbb N$. Let $\pmb{P}$ be a function $\pmb{P}: \mathbb{N} \times \mathbb{N} \in [0,1] $ where $\forall_i, \sum_{j=1}^\infty \pmb{P}(i,j) =1 $ and   $\pmb{P}(i)$ denote the $i$-th row of $\pmb{P}$ viewed as a distribution. 
    We define a function $\pmb{p'} : \mathbb{N} \rightarrow [0,1]$ by
    \begin{equation*}
        \pmb{p'}(j)=  \sum_{i =0}^\infty \pmb{p}(i)\pmb{P}(i,j)
    \end{equation*}
    Then, for any function $\pmb{v}:\mathbb{N}\rightarrow\mathbb R$,
    \begin{equation*}
        \sum_{i =0}^\infty  \pmb{p}(i) \|\pmb{v}\|_{P(i)} =\|\pmb{v}\|_{\pmb{p}'} .
    \end{equation*}
\end{lemma}
\begin{proof}[\Cref{lemm:weighted-l1-norm-stochatic-matrix}]
    We first show that $\pmb{p'} \in \Delta(\infty)$.
    Since $\pmb{p}(i) \in [0,1]$ and $\pmb{P}(i,j) \in [0,1]$, it follows that $\pmb{p'}(j) \in [0,1]$
    \begin{equation}
        \sum_{j=0}^\infty  \pmb{p}'(j) = \sum_{j=0}^\infty    \sum_{i =0}^\infty \pmb{p}(i)\pmb{P}(i,j) =  \sum_{i =0}^\infty \pmb{p}(i) \sum_{j=0}^\infty  \pmb{P}(i,j) = \sum_{i =0}^\infty \pmb{p}(i) =1 
    \end{equation}
    Thus $\pmb{p'} \in \Delta(\infty)$.

    Now we prove that the identity holds.
    By definition of the weighted $L_1$-norm,
    \begin{equation*}
        \sum_{i =0}^\infty  \pmb{p}(i) \|\pmb{v}\|_{\pmb{P}(i)}
        =  \sum_{j=0}^\infty    \sum_{i =0}^\infty \pmb{p}(i)\pmb{P}(i,j)|\pmb{v}(j)| 
        =  \sum_{j=0}^\infty  \pmb{p}'(j) |\pmb{v}(j)| 
        = \|\pmb{v}\|_{\pmb{p}'}.
    \end{equation*}
    It follows from the fact that $\pmb{p}^\top \pmb{P} = (\pmb{p}')^\top= \sum_{i =0}^\infty   \pmb{p}(i) \pmb{P}(i)$.
\end{proof}

\begin{lemma}
    \label{lemm:L1-Linf-inner-norm}
    For any two vectors $\pmb{v},\pmb{w} \in \realnum^\infty$ where $\norm{\pmb{v}}_1 \le \infty$ and $\norm{\pmb{w}}_\infty \le \infty$,
    \begin{align*}
        \big|\pmb{v}^\top \pmb{w}\big|=\|\pmb{w}\|_\infty  \|\pmb{v}\|_1 . 
    \end{align*}
\end{lemma}
\begin{proof}[\Cref{lemm:L1-Linf-inner-norm}]
    We can write
    \begin{align}
        \big|\pmb{v}^\top \pmb{w}\big|
        =|  \sum_{i = 1}^\infty  \pmb{v}(i)\pmb{w}(i)|
        \le   \sum_{i = 1}^\infty  |\pmb{v}(i)\pmb{w}(i)|
        &\le  \sum_{i = 1}^\infty  |\pmb{v}(i)||\pmb{w}(i)|\\
        &\le\max_{i \in \mathbb{N}} |\pmb{w}(i)|  \sum_{i = 1}^\infty |\pmb{v}(i)| &(\text{by~\cref{def:l1-linf-norms}})\\ 
        &=\|\pmb{w}\|_\infty  \|\pmb{v}\|_1 . 
    \end{align}
\end{proof}

\subsection{Value loss bounds for countably infinite state spaces}
\label{app:value-loss-bounds-infinity}

\begin{lemma}[One-step state value difference in countably infinite state spaces]
\label{lemm:one-step-value-function-difference-infty}
    Consider two natural number MRPs $\langle \pmb{P}, \pmb{r} \rangle$ and $\langle \pmb{P}', \pmb{r}' \rangle$ with value function vectors $\pmb{v}$ and $\pmb{v}'$ respectively.
    Then,
    \begin{equation}
        \forall i \in \natnum, ~ \abs{ \pmb{v}(i) - \pmb{v}'(i) } \le \abs{ \pmb{r}(i) - \pmb{r}'(i) } + \gamma \frac{\vmax}{2} \norm{ \pmb{P}(i) - \pmb{P}'(i) }_1 + \gamma \norm{ \pmb{v} - \pmb{v}' }_{ \pmb{P}(i) } .
    \end{equation}
\end{lemma}
\begin{proof}[\Cref{lemm:one-step-value-function-difference-infty}]
    First, we observe that $\pmb{P}(i) \in \probsimplex{\infty}$ and $\pmb{v} \in \realnum^\infty$.
    Then,
    \begin{eqnarray}
        \pmb{P}(i) \pmb{v} = \sum_{j=1}^\infty \pmb{P}(i,j) \pmb{v}(j) \le \infty
    \end{eqnarray}
    because $\pmb{v}(j) \in [0,\vmax]$ and the row vector $\pmb{P}(i)$ is a probability vector.
    Furthermore, we have that
    \begin{equation}
        \left( \pmb{P}(i) - \pmb{P}'(i) \right) \pmb{1} \frac{\vmax}{2} = \underbrace{\sum_{j=1}^\infty \left( \pmb{P}(i,j) - \pmb{P}'(i,j) \right)}_{=0} \frac{\vmax}{2} 
        = 0, \label{eq:value-shift}
    \end{equation}
    where $\pmb{1}$ is the vector of all ones.
    Then, using the Bellman equations for $\pmb{v}$ and $\pmb{v}'$ we have an arbitrary $i$ that
    \begin{align}
        & \abs{ \pmb{v}(i) - \pmb{v}'(i) } \nonumber \\
        &=   \abs{ ( \pmb{r}(i) + \gamma \pmb{P}(i) \pmb{v} ) - ( \pmb{r}'(i) + \gamma \pmb{P}'(i) \pmb{v}' ) } \\
        &\le \abs{ \pmb{r}(i) - \pmb{r}'(i) } + \gamma \abs{ \pmb{P}(i) \pmb{v} - \pmb{P}(i) \pmb{v}' + \pmb{P}(i) \pmb{v}' - \pmb{P}'(i) \pmb{v}' } \\
        &\le \abs{ \pmb{r}(i) - \pmb{r}'(i) } + \gamma \abs{ \pmb{P}(i) \pmb{v} - \pmb{P}(i) \pmb{v}' } + \gamma \abs{ \pmb{P}(i) \pmb{v}' - \pmb{P}'(i) \pmb{v}' } \\
        &=   \abs{ \pmb{r}(i) - \pmb{r}'(i) } + \gamma \abs{ \pmb{P}(i) ( \pmb{v} - \pmb{v}' ) } + \gamma \abs{ ( \pmb{P}(i) - \pmb{P}'(i) ) \pmb{v}' } \\
        &=   \abs{ \pmb{r}(i) - \pmb{r}'(i) } + \gamma \abs{ \pmb{P}(i) ( \pmb{v} - \pmb{v}' ) } + \gamma \abs{ ( \pmb{P}(i) - \pmb{P}'(i) ) ( \pmb{v}' - \pmb{1} \frac{\vmax}{2} ) } &(\text{\cref{eq:value-shift}}) \\
        &\le \abs{ \pmb{r}(i) - \pmb{r}'(i) } + \gamma \abs{ \pmb{P}(i) ( \pmb{v} - \pmb{v}' ) } + \gamma \norm{ \pmb{P}(i) - \pmb{P}'(i) }_1 \norm{ \pmb{v}' - \pmb{1} \frac{\vmax}{2} }_\infty &(\text{\cref{lemm:L1-Linf-inner-norm}}) \\
        &\le \abs{ \pmb{r}(i) - \pmb{r}'(i) } + \gamma \norm{ \pmb{v} - \pmb{v}' }_{ \pmb{P}(i) } + \gamma \norm{ \pmb{P}(i) - \pmb{P}'(i) }_1 \norm{ \pmb{v}' - \pmb{1} \frac{\vmax}{2} }_\infty &(\text{\cref{def:weighted-L1-norm}}) \\
        &\le \abs{ \pmb{r}(i) - \pmb{r}'(i) } + \gamma \norm{ \pmb{v} - \pmb{v}' }_{ \pmb{P}(i) } + \gamma \norm{ \pmb{P}(i) - \pmb{P}'(i) }_1 \frac{\vmax}{2} 
    \end{align}
    where the last line follows because every entry of $\pmb{v}'$ lies between zero and $\vmax$.
\end{proof}

\begin{lemma}[MRP Value Loss]
    \label{lemm:mrp-value-loss}
    Consider two natural number MRPs $\langle \pmb{P}, \pmb{r} \rangle$ and $\langle \pmb{P}', \pmb{r}' \rangle$ with value function vectors $\pmb{v}$ and $\pmb{v}'$ respectively.
    For the MRP $\langle \pmb{P}, \pmb{r} \rangle$ and a start-state distribution vector $\pmb{d}_0 \in \probsimplex{\infty}$, let $\pmb{d}_{\texttt{S}}$ be the normalized SR for MRP $\langle \pmb{P}, \pmb{r} \rangle$.\footnote{We omit the policy dependency because an MRP does not have any actions.}
    Then,
    \begin{equation*}
        \norm{ \pmb{v} - \pmb{v}' }_{\pmb{d}_0} \le \frac{1}{1 - \gamma} \left( \norm{ \pmb{r} - \pmb{r}' }_{ \pmb{d}_{\texttt{S}} } + \frac{\gamma}{2} \vmax \norm{ \pmb{P} - \pmb{P}' }_{ \pmb{d}_{\texttt{S}} } \right) .
    \end{equation*}
\end{lemma}

\begin{proof}[\Cref{lemm:mrp-value-loss}]
    First we prove a $T$-step version of~\cref{lemm:mrp-value-loss} by induction and then generalize this bound to infinite horizons.

    \textit{Induction Hypothesis:} 
    \begin{equation}
        \norm{ \pmb{v} - \pmb{v}' }_{\pmb{d}_0} \le \norm{ \pmb{r} - \pmb{r}' }_{ \sum_{t=0}^T \gamma^{t} \pmb{d}_t^\pi } + \frac{\gamma}{2} \vmax \norm{ \pmb{P} - \pmb{P}' }_{ \sum_{t=0}^T \gamma^{t} \pmb{d}_t^\pi } + \gamma^T \norm{ \pmb{v} - \pmb{v}' }_{ \pmb{d}_T } \label{eq:t-step-value-difference-infty}
    \end{equation}

    \textit{Base Case:} For $T=1$, we have that
    \begin{align}
        & \norm{ \pmb{v} - \pmb{v}' }_{\pmb{d}_0} \nonumber \\
        &= \sum_{i=1}^\infty \pmb{d}_0(i) \abs{ \pmb{v}(i) - \pmb{v}'(i) } \\
        &\le \sum_{i=1}^\infty \pmb{d}_0(i) \left( \abs{ \pmb{r}(i) - \pmb{r}'(i) } + \frac{\gamma}{2} \vmax \norm{ \pmb{P}(i) - \pmb{P}'(i) }_1 + \gamma \norm{ \pmb{v} - \pmb{v}' }_{ \pmb{P}(i) } \right) &(\text{\cref{lemm:one-step-value-function-difference-infty}}) \\
        &=   \norm{ \pmb{r} - \pmb{r}' }_{\pmb{d}_0} + \frac{\gamma}{2} \vmax \norm{ \pmb{P} - \pmb{P}' }_{ \pmb{d}_0 } + \gamma \norm{ \pmb{v} - \pmb{v}' }_{ \sum_{i=1}^\infty \pmb{d}_0(i) \pmb{P}(i) } &(\text{\cref{def:weighted-L1-norm}}) \\
        &=   \norm{ \pmb{r} - \pmb{r}' }_{\pmb{d}_0} + \frac{\gamma}{2} \vmax \norm{ \pmb{P} - \pmb{P}' }_{ \pmb{d}_0 } + \gamma \norm{ \pmb{v} - \pmb{v}' }_{ \pmb{d}_1 }, \label{eq:one-step-value-diff-recursion}
    \end{align}
    because $\pmb{d}_1 = \sum_{i=1}^\infty \pmb{d}_0(i) \pmb{P}(i)$.

    \textit{Induction Step:} Applying the bound in~\cref{eq:t-step-value-difference-infty} for a start distribution vector $\pmb{d}_{T}$, we have
    \begin{align}
        \norm{ \pmb{v} - \pmb{v}' }_{\pmb{d}_0} 
        &\le \norm{ \pmb{r} - \pmb{r}' }_{ \sum_{t=0}^T \gamma^{t} \pmb{d}_t^\pi } + \frac{\gamma}{2} \vmax \norm{ \pmb{P} - \pmb{P}' }_{ \sum_{t=0}^T \gamma^{t} \pmb{d}_t^\pi } + \gamma^T \norm{ \pmb{v} - \pmb{v}' }_{ \pmb{d}_T } \\
        &\le \norm{ \pmb{r} - \pmb{r}' }_{ \sum_{t=0}^T \gamma^{t} \pmb{d}_t^\pi } + \frac{\gamma}{2} \vmax \norm{ \pmb{P} - \pmb{P}' }_{ \sum_{t=0}^T \gamma^{t} \pmb{d}_t^\pi } \nonumber \\
        &~~~~~+ \gamma^T \norm{ \pmb{r} - \pmb{r}' }_{\pmb{d}_T} + \gamma^{T} \frac{\gamma}{2} \vmax \norm{ \pmb{P} - \pmb{P}' }_{ \pmb{d}_T } + \gamma^{T+1} \norm{ \pmb{v} - \pmb{v}' }_{ \pmb{d}_{T+1} } \\
        &=   \norm{ \pmb{r} - \pmb{r}' }_{ \sum_{t=0}^{T+1} \gamma^{t} \pmb{d}_t^\pi } + \frac{\gamma}{2} \vmax \norm{ \pmb{P} - \pmb{P}' }_{ \sum_{t=0}^{T+1} \gamma^{t} \pmb{d}_t^\pi } + \gamma^{T+1} \norm{ \pmb{v} - \pmb{v}' }_{ \pmb{d}_{T+1} }
    \end{align}

    \textit{Final Bound:} Next, we take the limit where $T \to \infty$ of the $T$-step bound in~\cref{eq:t-step-value-difference-infty}.
    First, we note that for the limits
    \begin{align}
        \lim_{T \to \infty}  \norm{ \pmb{r} - \pmb{r}' }_{ \sum_{t=0}^T \gamma^{t} \pmb{d}_t^\pi } = \norm{ \pmb{r} - \pmb{r}' }_{ \sum_{t=0}^\infty \gamma^{t} \pmb{d}_t^\pi } \\
        \lim_{T \to \infty}  \norm{ \pmb{P} - \pmb{P}' }_{ \sum_{t=0}^T \gamma^{t} \pmb{d}_t^\pi } = \norm{ \pmb{P} - \pmb{P}' }_{ \sum_{t=0}^\infty \gamma^{t} \pmb{d}_t^\pi }
    \end{align}
    because increasing $T$ adds additional weight terms.
    Moreover, we can view the terms $\norm{ \pmb{r} - \pmb{r}' }_{ \sum_{t=0}^T \gamma^{t} \pmb{d}_t^\pi }$ and $\norm{ \pmb{P} - \pmb{P}' }_{ \sum_{t=0}^T \gamma^{t} \pmb{d}_t^\pi }$ as sequences that are strictly non-decreasing in $T$ and therefore
    \begin{align}
        \forall_{T \ge 0},&&  \norm{ \pmb{r} - \pmb{r}' }_{ \sum_{t=0}^T \gamma^{t} \pmb{d}_t^\pi } &\le \norm{ \pmb{r} - \pmb{r}' }_{ \sum_{t=0}^\infty \gamma^{t} \pmb{d}_t^\pi } \\
        \forall_{T \ge 0},&&  \norm{ \pmb{P} - \pmb{P}' }_{ \sum_{t=0}^T \gamma^{t} \pmb{d}_t^\pi } &\le \norm{ \pmb{P} - \pmb{P}' }_{ \sum_{t=0}^\infty \gamma^{t} \pmb{d}_t^\pi }.
    \end{align}
    This allows us to obtain the limit as $T \to \infty$ and an upper bound:
    \begin{align}
         \norm{ \pmb{v} - \pmb{v}' }_{\pmb{d}_0} 
         &\le \lim_{T \to \infty} \left[ \norm{ \pmb{r} - \pmb{r}' }_{ \sum_{t=0}^T \gamma^{t} \pmb{d}_t^\pi } + \frac{\gamma}{2} \vmax \norm{ \pmb{P} - \pmb{P}' }_{ \sum_{t=0}^T \gamma^{t} \pmb{d}_t^\pi } + \gamma^T \norm{ \pmb{v} - \pmb{v}' }_{ \pmb{d}_T } \right] \\
         &= \norm{ \pmb{r} - \pmb{r}' }_{ \sum_{t=0}^\infty \gamma^{t} \pmb{d}_t^\pi } + \frac{\gamma}{2} \vmax \norm{ \pmb{P} - \pmb{P}' }_{ \sum_{t=0}^\infty \gamma^{t} \pmb{d}_t^\pi } \\
         &= \frac{1}{1 - \gamma} \left (\norm{ \pmb{r} - \pmb{r}' }_{ (1 - \gamma) \sum_{t=0}^\infty \gamma^{t} \pmb{d}_t^\pi } + \frac{\gamma}{2} \vmax \norm{ \pmb{P} - \pmb{P}' }_{ (1 - \gamma) \sum_{t=0}^\infty \gamma^{t} \pmb{d}_t^\pi } \right )\\
         &= \frac{1}{1 - \gamma} \left ( \norm{ \pmb{r} - \pmb{r}' }_{ \pmb{d}_{\texttt{S}}^\pi } + \frac{\gamma}{2} \vmax \norm{ \pmb{P} - \pmb{P}' }_{ \pmb{d}_{\texttt{S}}^\pi } \right),
    \end{align}
    because $\pmb{d}_{\texttt{S}}^\pi = (1 - \gamma) \sum_{t=0}^\infty \gamma^{t} \pmb{d}_t^\pi$.
\end{proof}

\textbf{Restatement of \cref{prop:simulation-lemma}}.
    \getkeytheorem[body]{prop:simulation-lemma} 
\begin{proof}[Proof of~\cref{prop:simulation-lemma}]
    We observe that $\norm{ \pmb{r}^\pi - \widehat{\pmb{r}}^\pi }_{ \widehat{\pmb{d}}^\pi_{\texttt{S}} } \le \eps_r$ and $\norm{ \pmb{P}^\pi - \widehat{\pmb{P}}^\pi }_{ \widehat{\pmb{d}}^\pi_{\texttt{S}} } \le \eps_p$ and apply~\cref{lemm:mrp-value-loss}.
\end{proof}

\subsubsection{Proof of~\cref{prop:value-loss-learning-finite-steps}}
\label{app:value-loss-learning-finite-steps-proof}

First the following lemma is proven on the learned reward and transition functions.

\begin{lemma}[Reward and transition bound]
    \label{lemm:one-step-reward-and-transition-bound}
    We consider a natural number MDP $M^\natnum$ with all rewards lying in the interval $[0,\rmax]$.
    Then for the optimal policy $\pi^\star$ and the learned $\widehat{\pi}$ (\cref{asmpt:learning-finite-steps}), the corresponding MRPs will be $\langle \pmb{r}^{\pi^\star}, \pmb{P}^{\pi^\star} \rangle$, $\langle \pmb{r}^{\widehat{\pi}}, \pmb{P}^{\widehat{\pi}} \rangle$.

    For these MRPs, we have that for all $s \in \natnum$,
    \begin{equation}
        \abs{ \pmb{r}^{\pi^\star}(s)-\pmb{r}^{\widehat{\pi}}(s) } \le
        \begin{cases}
            \eps \rmax & s \in \statespace_\obs \\
            \rmax      & s \in \statespace_\unk
        \end{cases}
        \label{eq:piecewise-reward-bound}
    \end{equation}
    and 
    \begin{equation}
        \norm{ \pmb{P}^{\pi^\star}(s)-\pmb{P}^{\widehat{\pi}}(s) }_1 \le
        \begin{cases}
            \eps & s \in \statespace_\obs \\
            2    & s \in \statespace_\unk
        \end{cases}
        \label{eq:piecewise-transition-bound}
    \end{equation}
\end{lemma}
\begin{proof}[\Cref{lemm:one-step-reward-and-transition-bound}]
    First we derive the reward bound (\cref{eq:piecewise-reward-bound}) and then the transition bound (\cref{eq:piecewise-transition-bound}).
    
    \textit{Reward bound derivation.}
    We have that
    \begin{equation}
        \pmb{r}^\pi(s) = \sum_{a\in \actionspace} \pi(a| s) r(s,a)\quad s\in \statespace
        \label{eq:mrp_reward_def}
    \end{equation}
    Using~\cref{eq:mrp_reward_def}, for any $s\in \statespace$,
    \begin{equation}
        \pmb{r}^{\pi^\star}(s)-\pmb{r}^{\widehat{\pi}}(s)
        = \sum_{a\in \actionspace}\big(\pi^\star(a|s)-\widehat{\pi}(a|s)\big)r(s,a).
        \label{eq:reward-diff-identity}
    \end{equation}
    For $s\in \statespace_\obs$, $\| \pmb{\pi}^\star_s - \widehat{\pmb{\pi}}_s \|_1\le \eps$. 
    We take absolute values in~\cref{eq:reward-diff-identity} and use $r(s,a) \in [0 ,\rmax]$:
    \begin{align}
        | \pmb{r}^{\pi^\star}(s)-\pmb{r}^{\widehat{\pi}}(s) |
        &\le \sum_{ a\in \actionspace } | \pi^\star(a|s) - \widehat{\pi}(a|s) | r(s,a) \nonumber \\
        &\le \rmax\sum_{a\in \actionspace}|\pi^\star(a| s)-\widehat{\pi}(a|s) | 
        = \rmax \| \pmb{\pi}^\star_s - \widehat{\pmb{\pi}}_s \|_1 \nonumber 
        \le \eps\rmax
    \end{align}
    For $s\in \statespace_\unk$, the policy $\widehat{\pi}$ selects action arbitrarily and therefore the one-step reward loss can be at most $\rmax$ and
    \begin{equation}
        \pmb{r}^{\pi^\star}(s)-\pmb{r}^{\widehat{\pi}}(s) \le \rmax.
    \end{equation}
    
    \textit{Transition bound derivation.}
    For a fixed state $s$, we define the row-wise absolute difference between two MRP transition functions by
    \begin{equation}
        \norm{ \pmb{P}^{\pi^\star}(s)-\pmb{P}^{\widehat{\pi}}(s) }_1 
        = \sum_{i = 1}^\infty \abs{ \pmb{P}^{\pi^\star}(s,i)-\pmb{P}^{\widehat{\pi}}(s,i) } \label{eq:row-wise-absolute-diff}
    \end{equation}
    because we assume that the state space is the natural numbers $\natnum$ and we consider the L1-norm on infinite-dimensional vectors.
    Because~\cref{eq:row-wise-absolute-diff} only computes the L1-norm for the difference between two probability vectors, the term in~\cref{eq:row-wise-absolute-diff} is bounded.
    Moreover, we have that
    \begin{align}
        \norm{ \pmb{P}^{\pi^\star}(s)-\pmb{P}^{\widehat{\pi}}(s) }_1 
        &=   \sum_{i = 1}^\infty \abs{ \pmb{P}^{\pi^\star}(s,i)-\pmb{P}^{\widehat{\pi}}(s,i) } \nonumber \\
        &=   \sum_{i = 1}^\infty \abs{ \sum_{a \in \actionspace} \pi^\star(a|s) \pmb{P}^a(s,i) - \widehat{\pi}(a|s) \pmb{P}^a(s,i) } \nonumber \\
        &=   \sum_{i = 1}^\infty \abs{ \sum_{a \in \actionspace} ( \pi^\star(a|s) - \widehat{\pi}(a|s) ) \pmb{P}^a(s,i) } \nonumber \\
        &\le \sum_{i = 1}^\infty       \sum_{a \in \actionspace} \abs{ \pi^\star(a|s) - \widehat{\pi}(a|s) } \abs{ \pmb{P}^a(s,i) } \nonumber \\
        &=   \underbrace{\sum_{a \in \actionspace} \abs{ \pi^\star(a|s) - \widehat{\pi}(a|s) }}_{=\norm{ \pmb{\pi}^\star_s - \widehat{\pmb{\pi}}_s }_1} \underbrace{\sum_{i = 1}^\infty \abs{ \pmb{P}^a(s,i) }}_{ =\norm{\pmb{P}^a(s)}_1 } \nonumber \\
        &=   \norm{ \pmb{\pi}^\star_s - \widehat{\pmb{\pi}}_s }_1 \underbrace{\norm{\pmb{P}^a(s)}_1 }_{=1} \nonumber \\
        &=   \norm{ \pmb{\pi}^\star_s - \widehat{\pmb{\pi}}_s }_1 \label{eq:transition-bound}
    \end{align}
    Because we assume that the policy $\widehat{\pi}$ is $\eps$-optimal for every state $s \in \statespace_\obs$ (\cref{asmpt:learning-finite-steps}), we have that
    \begin{equation}
        \forall_{s \in \statespace_\obs},~ \norm{ \pmb{P}^{\pi^\star}(s)-\pmb{P}^{\widehat{\pi}}(s) }_1 \le \norm{ \pmb{\pi}^\star_s - \widehat{\pmb{\pi}}_s }_1 \le \eps.
    \end{equation}
    For states $s \in \statespace_\unk$, the policy $\widehat{\pi}$ can select actions by sampling an arbitrary probability vector. 
    In this case, the term $\norm{ \pmb{\pi}^\star_s - \widehat{\pmb{\pi}}_s }_1$ because the difference between two arbitrary probability vectors can be at most two.
    Therefore,
    \begin{equation}
        \forall_{s \in \statespace_\unk},~ \norm{ \pmb{P}^{\pi^\star}(s)-\pmb{P}^{\widehat{\pi}}(s) }_1 \le \norm{ \pmb{\pi}^\star_s - \widehat{\pmb{\pi}}_s }_1 \le 2.
    \end{equation}
\end{proof}

\noindent \textbf{Restatement of~\cref{prop:value-loss-learning-finite-steps}.}
    \getkeytheorem[body]{prop:value-loss-learning-finite-steps}
\begin{proof}[\Cref{prop:value-loss-learning-finite-steps}]
    The proof is an application of~\cref{lemm:mrp-value-loss} to the history $M^\histories$ under start state distribution $\pmb{d}_0$ where the scaled SR can be picked for either MRP (either $\pmb{d}^{\widehat{\pi}}$ or $\pmb{d}^{{\pi}^{\star}}$).
    Because the history MDP is bisimilar to a natural number MDP, we observe that $J(\pi^\star,\pmb{d}_0) - J(\widehat{\pi},\pmb{d}_0) = \norm{ \pmb{v}^{\pi*} - \pmb{v}^{\widehat{\pi} }}_{\pmb{d}_0}$.
    By~\cref{lemm:one-step-reward-and-transition-bound}, we can write the reward difference term for any probability vector $\pmb{d}$ as
    \begin{align}
        \norm{ \pmb{r} - \pmb{r}' }_{ \pmb{d} } 
        &= \sum_{i \in \natnum} \pmb{d}(i) \abs{ \pmb{r}(i) - \pmb{r}'(i) } \\
        &= \sum_{i \in \statespace_\obs} \pmb{d}(i) \abs{ \pmb{r}(i) - \pmb{r}'(i) } + \sum_{i \in \statespace_\unk} \pmb{d}(i) \abs{ \pmb{r}(i) - \pmb{r}'(i) } \\
        &\le \sum_{i \in \statespace_\obs} \pmb{d}(i) \eps \rmax + \sum_{i \in \statespace_\unk} \pmb{d}(i) \rmax \\ 
        &= \pmb{d}( \statespace_\obs ) \eps \rmax + \pmb{d}( \statespace_\unk ) \rmax . \label{eq:reward-diff-bnd-proof}
    \end{align}
    Similarly, we have for the transition matrix difference that
    \begin{align}
        \norm{ \pmb{P} - \pmb{P}' }_{ \pmb{d} } 
        &= \sum_{i \in \natnum} \pmb{d}(i) \norm{ \pmb{P}(i) - \pmb{P}'(i) }_1 \\
        &= \sum_{i \in \statespace_\obs} \pmb{d}(i) \norm{ \pmb{P}(i) - \pmb{P}'(i) }_1 + \sum_{i \in \statespace_\unk} \pmb{d}(i) \norm{ \pmb{P}(i) - \pmb{P}'(i) }_1  \\
        &= \sum_{i \in \statespace_\obs} \pmb{d}(i) \eps + \sum_{i \in \statespace_\unk} \pmb{d}(i) 2  \\
        &= \pmb{d}( \statespace_\obs ) \eps + \pmb{d}( \statespace_\unk ) 2 . \label{eq:transition-diff-bnd-proof}
    \end{align}
    The final results follows by substituting~\cref{eq:reward-diff-bnd-proof,eq:transition-diff-bnd-proof} into~\cref{lemm:mrp-value-loss} for either $\pmb{d}^{\widehat{\pi}}$ or $\pmb{d}^{{\pi}^{\star}}$ to obain
    \begin{align}
         J(\pi^\star,\pmb{d}_0) - J(\widehat{\pi},\pmb{d}_0) 
        &\le \frac{1}{1 - \gamma} \left( \pmb{d}(\histories_\obs) \eps \left( \rmax + \frac{\gamma}{2} \vmax \right) + \pmb{d}(\histories_\unk) \left( \rmax + \frac{\gamma}{2} \vmax 2 \right) \right) \nonumber \\
        &\le \frac{1}{1 - \gamma} \left( \pmb{d}(\histories_\obs) \eps \left( \rmax + \frac{\gamma}{2} \vmax \right) + \pmb{d}(\histories_\unk) \vmax \right) .
    \end{align}
\end{proof}

\subsection{Proof of the telescoping performance bound for countably infinite states}
\label{app:telescoping-perf-diff-proof}

\begin{lemma}[Discounted Return]
    \label{lemm:discounted-return}
    Consider any start state distribution $\pmb{d}_0 \in \probsimplex{\infty}$ and let $\pmb{d}^\pi \in \probsimplex{\infty}$ be the $\gamma$-discounted state distribution induced by the start state distribution $\pmb{d}_0$ and the policy $\pi$ for a natural number MDP $M^\natnum = \langle \natnum, \actionspace, p_\natnum, r_\natnum, \gamma \rangle$.
    Then,
    \begin{equation}
        J(\pi) = \frac{1}{1 - \gamma} \expec{ (s,a) \sim \pmb{d}^\pi, r \sim r_\natnum(s,a) }{ r }
    \end{equation}
\end{lemma}
\begin{proof}[Proof of~\cref{lemm:discounted-return}]
    First, we observe that 
    \begin{equation}
        \pmb{d}^\pi(s,a) = (1-\gamma) \sum_{t=1}^\infty \gamma^{t-1} \pmb{d}_t(s) \pi(a|s).
    \end{equation}
    Then, the total return is then
    \begin{align}
        J(\pi) 
        &= \expec{\pmb{d}_0, \pi}{ \sum_{t=1}^\infty \gamma^{t - 1} r_t } \\
        &= \expec{\pmb{d}_0, \pi}{ \sum_{t=1}^\infty \gamma^{t - 1} \expec{r \sim r_\natnum(s_t,a_t)}{r} } \\
        &= \frac{1 - \gamma}{1 - \gamma} \sum_{t=1}^\infty \gamma^{t - 1} \expec{s_t \sim \pmb{d}_t, a_t \sim \pi}{ \expec{r \sim r_\natnum(s_t,a_t)}{r} } 
        = \frac{1}{1 - \gamma} \expec{(s,a) \sim \pmb{d}, r \sim r_\natnum(s_t,a_t)}{ r } .
    \end{align}
\end{proof}

\begin{lemma}[Operator Difference]
    \label{lemm:operator-difference}
    For a natural number MDP $M^\natnum$ with state-action bijection $k: \natnum \times \actionspace \to \natnum$, we have $\forall \pmb{f} \in \realnum^\infty$ and start distributions $\pmb{d}_0 \in \probsimplex{\infty}$ that
    \begin{equation}
        \expec{ s \sim \pmb{d}_0, a \sim \pi }{\pmb{f}(k(s,a))} - J(\pi) = \frac{1}{1 - \gamma} \expec{(s,a) \sim \pmb{d}^\pi}{ (\pmb{f} - \bellmanopinf^\pi_{M^\natnum} \pmb{f}) (k(s,a)) }.
    \end{equation}
\end{lemma}
\begin{proof}[Proof of~\cref{lemm:operator-difference}]
  
    \begin{align}
        & \frac{1}{1 - \gamma} \expec{(s,a) \sim \pmb{d}^\pi}{ (\pmb{f} - \bellmanopinf^\pi_{M^\natnum} \pmb{f}) (k(s,a)) } \nonumber \\
        &=\sum_{t=1}^\infty \gamma^{t-1} \expec{}{ (\pmb{f} - \bellmanopinf^\pi_{M^\natnum} \pmb{f}) (k(s_t,a_t)) }\\
        &=\sum_{t=1}^\infty \gamma^{t-1} \expec{}{ \pmb{f} (k(s_t,a_t)) - r_t - \gamma \pmb{f} (k(s_{t+1},a_{t+1}))}\\
        &=\sum_{t=1}^\infty \gamma^{t-1} \expec{}{ \pmb{f} (k(s_t,a_t))} -\sum_{t=1}^\infty \gamma^{t-1} \expec{}{r_t} - \sum_{t=1}^\infty \gamma^{t-1} \expec{}{\gamma \pmb{f} (k(s_{t+1},a_{t+1}))}\\
        &=\sum_{t=1}^\infty \gamma^{t-1} \expec{}{ \pmb{f} (k(s_t,a_t))}  - \sum_{t=2}^\infty \gamma^{t-1} \expec{}{\gamma \pmb{f} (k(s_{t},a_{t}))}-\sum_{t=1}^\infty \gamma^{t-1} \expec{}{r_t}\\
         &= \expec{}{ \pmb{f} (k(s_1,a_1))} -\sum_{t=1}^\infty \gamma^{t-1} \expec{}{r_t}
    \end{align}
    Then by using the definition of discounted sum and $ s_1 \sim \pmb{d}_0, a_1 \sim \pi(\cdot|s_1) $ we can write,
    \begin{equation}
        \expec{}{ \pmb{f} (k(s_1,a_1))}=\expec{ s \sim \pmb{d}_0, a \sim \pi }{\pmb{f}(k(s,a))}.
    \end{equation}
    Combining these, we finally have,
    \begin{equation}
        \frac{1}{1 - \gamma} \expec{(s,a) \sim \pmb{d}^\pi}{ (\pmb{f} - \bellmanopinf^\pi_{M^\natnum} \pmb{f}) (k(s,a)) }=  \expec{ s \sim \pmb{d}_0, a \sim \pi }{\pmb{f}(k(s,a))} - J(\pi).
    \end{equation}
\end{proof}

\begin{lemma}[Telescoping Performance Difference]
    \label{lemm:policy-perf-diff}
    For any vector $\pmb{f} \in \realnum^\infty$, state-action bijection $k: \natnum \times \actionspace \to \natnum$, and start state distribution $\pmb{d}_0 \in \probsimplex{\infty}$, we have for two policies $\pi$ and $\pi_{\pmb{f}}$ that
    \begin{align*}
        J(\pi, \pmb{d}_0) - J(\pi_{\pmb{f}}, \pmb{d}_0) 
        \le \frac{1}{1 - \gamma} \Big( & \expec{ (s,a) \sim \pmb{d}^{\pi} }{ (\bellmanopinf^{\pi_{\pmb{f}}}_{M^\natnum} \pmb{f} - \pmb{f})(k(s,a)) } \\
                                       & - \expec{ (s,a) \sim \pmb{d}^{\pi_{\pmb{f}}} }{ (\pmb{f} - \bellmanopinf^{\pi_{\pmb{f}}}_{M^\natnum} \pmb{f})(k(s,a)) } \Big),
    \end{align*}
    where $\pi_{ \pmb{f} }(s) = \arg \max_a \pmb{f}(k(s,a))$.
    Moreover, we have that
    \begin{equation}
        J(\pi, \pmb{d}_0) - J(\pi_{\pmb{f}}, \pmb{d}_0)  
        \le \frac{1}{1 - \gamma}
        \Big(
              \norm{ \pmb{f} - \bellmanopinf^{\pi_{\pmb{f}}}_{M^\natnum} \pmb{f} }_{\pmb{d}^\pi}
            + \norm{ \pmb{f} - \bellmanopinf^{\pi_{\pmb{f}}}_{M^\natnum} \pmb{f} }_{\pmb{d}^{\pi_{\pmb{f}}}}
        \Big)
    \end{equation}
\end{lemma}
\begin{proof}[Proof of~\cref{lemm:policy-perf-diff}]
    To simplify the notation, we omit the start state distribution $\pmb{d}_0$ form the $J(\cdot)$ function and observe
    \begin{align}
        J(\pi) - J(\pi_{\pmb{f}}) 
        &= J(\pi) - \expec{s \sim \pmb{d}_0, a \sim \pi}{ \pmb{f}(k(s,a)) } + \expec{s \sim \pmb{d}_0, a \sim \pi}{ \pmb{f}(k(s,a)) } - J(\pi_{\pmb{f}}) \\
        &\le J(\pi) - \expec{s \sim \pmb{d}_0, a \sim \pi}{ \pmb{f}(k(s,a)) } + \expec{s \sim \pmb{d}_0, a \sim \pi_{\pmb{f}}}{ \pmb{f}(k(s,a)) } - J(\pi_{\pmb{f}}) \label{eq:perf-diff-proof-1}
    \end{align}
    because $\pi_{\pmb{f}}(s) = \arg \max_{a} \pmb{f}(k(s,a))$ and therefore 
    \begin{equation}
        \expec{s \sim \pmb{d}_0, a \sim \pi}{ \pmb{f}(k(s,a)) } \le \expec{s \sim \pmb{d}_0, a \sim \pi_{\pmb{f}}}{ \pmb{f}(k(s,a)) }.
    \end{equation}
    We apply~\cref{lemm:operator-difference} twice for each term in~\cref{eq:perf-diff-proof-1} and obtain
    \begin{align}
        J(\pi) - J(\pi_{\pmb{f}}) 
        \le & \frac{1}{1 - \gamma} \expec{(s,a) \sim \pmb{d}^\pi}{ (\bellmanopinf^\pi_{M^\natnum} \pmb{f} - \pmb{f})(k(s,a)) } \nonumber \\
            & + \frac{1}{1 - \gamma} \expec{(s,a) \sim \pmb{d}^{\pi_{\pmb{f}}}}{ (\pmb{f} - \bellmanopinf^{\pi_{\pmb{f}}}_{M^\natnum} \pmb{f})(k(s,a)) } .
    \end{align}
    Using again the fact that $\pi_{ \pmb{f} }(s) = \arg \max_a \pmb{f}(k(s,a))$, we observe that
    \begin{equation}
        \bellmanopinf^{ \pi }_{M^\natnum} \pmb{f} \le \bellmanopinf^{ \pi_{\pmb{f}} }_{M^\natnum} \pmb{f}
    \end{equation}
    and therefore
    \begin{align}
        J(\pi) - J(\pi_{\pmb{f}}) 
        \le & \frac{1}{1 - \gamma} \expec{(s,a) \sim \pmb{d}^\pi}{ (\bellmanopinf^{\pi_{\pmb{f}}}_{M^\natnum} \pmb{f} - \pmb{f})(k(s,a)) } \nonumber \\
            & + \frac{1}{1 - \gamma} \expec{(s,a) \sim \pmb{d}^{\pi_{\pmb{f}}}}{ (\pmb{f} - \bellmanopinf^{\pi_{\pmb{f}}}_{M^\natnum} \pmb{f})(k(s,a)) } . \label{eq:perf-diff-proof-2}
    \end{align}
    Applying the distribution weighted L1 norm to the expected values in~\cref{eq:perf-diff-proof-2} leads to the second bound:
    \begin{equation}
        J(\pi) - J(\pi_{\pmb{f}}) 
        \le \frac{1}{1 - \gamma}
        \Big(
              \norm{ \pmb{f} - \bellmanopinf^{\pi_{\pmb{f}}}_{M^\natnum} \pmb{f} }_{\pmb{d}^\pi}
            + \norm{ \pmb{f} - \bellmanopinf^{\pi_{\pmb{f}}}_{M^\natnum} \pmb{f} }_{\pmb{d}^{\pi_{\pmb{f}}}}
        \Big).
    \end{equation}
\end{proof}

\subsection{Projection operators and model reductions}
\label{app:projection-operators-and-model-reduction}

Since the history MDP $M^\histories$ is bisimilar to a natural-number MDP $M^\natnum$, we define a state abstraction as a function $\phi: \natnum \to \statespace^\phi$.
The definitions presented in this section easily generalize to history MDPs and infinite state POMDPs.
To connect quantities defined in the abstract MDP with those in the natural-number state space, we introduce up- and down-projection operators.

The \emph{down-projection operators} $\Phi^{\downarrow}$ map vectors and matrices from the abstract MDP to the natural number MDP.
They are defined as
\begin{align}
    \forall_{(i,a) \in \natnum \times \actionspace},&& \left( {\Phi^\downarrow} \pmb{v}_\phi \right)(k(i,a)) &= \pmb{v}_\phi( \phi(i),a ) &\text{and}&& \left( {\Phi^\downarrow} \pmb{v}_\phi \right) &\in \realnum^\infty, \label{eq:down-projection-vec} \\
    \forall_{(i,a,s_\phi) \in \natnum \times \actionspace \times \statespace^\phi},&& \left( {\Phi^\downarrow} \pmb{P}_\phi \right)(k(i,a),s_\phi) &= \pmb{P}_\phi( (\phi(i),a),s_\phi ) &\text{and}&& \left( {\Phi^\downarrow} \pmb{P}_\phi \right) &\in \realnum^{\infty \times \abs{\statespace^\phi}}. \label{eq:down-projection-mat} 
\end{align}
Although we use the same symbol $\Phi^{\downarrow}$ in \cref{eq:down-projection-vec,eq:down-projection-mat} to improve readability, the down-projection operators are defined differently for vectors and matrices.

The \emph{up-projection operators} $\Phi^{\uparrow}$ map an infinite-dimensional probability vector $\pmb{p}$ or transition matrix $\pmb{P}$ from the natural number state space (countably infinite ground state space) into the abstract state space.
They are defined as
\begin{align}
    \forall_{(s_\phi,a) \in \statespace^\phi \times \actionspace},&& \left( {\Phi^\uparrow} \pmb{p} \right)( s_\phi ) &= \sum_{i | \phi(i) = s_\phi } \pmb{p}(k(i,a)) &\text{and}&& \left( {\Phi^\uparrow} \pmb{p} \right) &\in \realnum^{\abs{\statespace^\phi}}, \label{eq:up-projection-vec} \\
    \forall_{(i,s_\phi) \in \natnum \times \statespace^\phi},&& \left( {\Phi^\uparrow} \pmb{P} \right)( i, s_\phi ) &= \sum_{j | \phi(j) = s_\phi } \pmb{P}(i,j) &\text{and}&& \left( {\Phi^\uparrow} \pmb{P} \right) &\in \realnum^{\infty \times \abs{\statespace^\phi}}. \label{eq:up-projection-mat}
\end{align}
As before,~\Cref{eq:up-projection-vec,eq:up-projection-mat} use the same symbol $\Phi^{\uparrow}$ to improve readability and note that the up-projection operators are defined differently for vectors and matrices.
Depending on the abstraction function $\phi$, we note that the summations in~\cref{eq:up-projection-vec,eq:up-projection-mat} may be an infinite series that may not converge.
Here, we restrict the definition of the up projection to probability vectors and stochastic matrices.
Thus, the sum of all entries in~\cref{eq:up-projection-vec,eq:up-projection-mat} is upper bounded by one.

Using these projection operators, we can generalize the approximate model reduction definitions presented by~\citet{abel2016OptimalBehaviorApproximate} to countably infinite states.
An approximate model reduction preserves one-step rewards and transitions as much as possible: 
Suppose one-step rewards are stored in a state-action indexed vector $\pmb{r} \in \realnum^\infty$, and the abstract MDP's one step rewards are stored in a vector $\pmb{r}_\phi \in \realnum^{ \abs{\statespace^\phi \times \actionspace} }$.
If for every state-action pair $(i,a)$ 
\begin{equation}
    \pmb{r}(k(i,a)) \approx (\Phi^{\downarrow} \pmb{r}_\phi)(k(i,a)) = \pmb{r}_\phi(\phi(i),a), \label{eq:approx-model-reduction-reward}
\end{equation}
then we can use the abstraction function $\phi$ and abstract rewards $\pmb{r}_\phi$ to predict one-step rewards in the original task.
In this case, the state abstraction $\phi$ is said to approximately preserve one-step rewards.
Transition probabilities are preserved by ensuring that the (ground) state to state partition transition probability approximately matches the corresponding abstract state to abstract state transition probabilities~\citep{givan2003EquivalenceNotionsModel,li2006towards,abel2016OptimalBehaviorApproximate}.
Specifically, if the infinite-dimensional transition matrix of the natural number MDP $M^\natnum$ is $\pmb{P} \in \realnum^{\infty \times \infty}$ and the finite-dimensional transition matrix of the abstract MDP $M^\phi$ is $\pmb{P}_\phi \in \realnum^{ \abs{\statespace^\phi \times \actionspace} \times \abs{\statespace^\phi} }$, then
\begin{equation}
    \underbrace{
        (\Phi^{\uparrow} \pmb{P})(k(i,a),s_\phi) =  \sum_{j | \phi(j) = s_\phi } \pmb{P}(k(i,a),j)
    }_{\text{state to state partition transition prob.}}
    \approx 
    \underbrace{
        \vphantom{\sum_{j | \phi(j) = s_\phi }} (\Phi^{\downarrow} \pmb{P}_\phi)(k(i,a),s_\phi) = \pmb{P}_\phi((\phi(i),a),s_\phi) .
    }_{\text{abstract state to abstract state transition prob.}}
    \label{eq:approx-model-reduction-transition}
\end{equation}
A successor-weighted model reduction is then constructed such that the differences between the left and right sides in~\cref{eq:approx-model-reduction-reward,eq:approx-model-reduction-transition} are bounded as much as possible with respect to some normalized state-action SR, as stated in~\cref{def:successor-weighted-model-reduction}.

\subsection{Proof of OOD generalization bound}
\label{app:ood-pomdp-thm}

\begin{proposition}[Associativity of Infinite Abstraction Projections]
    \label{prop:infinite-abstract-projection-associativity}
    Consider an abstraction function $\phi: \natnum \to \statespace^\phi$ mapping the natural number to a finite set of abstract states and two arbitrary vectors $\pmb{v},\pmb{v}' \in \realnum^\infty$ to a vector $\pmb{v}_\phi, \pmb{v}_\phi' \in \realnum^{\abs{\statespace^\phi}}$ and an arbitrary scalar $\alpha \in \realnum$.
    Then,
    \begin{enumerate}[nosep,topsep=0pt,leftmargin=2em]
        \item $\alpha {\Phi^\downarrow} \pmb{v}_\phi = {\Phi^\downarrow} (\alpha \pmb{v}_\phi)$,
        \item $\alpha {\Phi^\uparrow} \pmb{v} = {\Phi^\uparrow} ( \alpha \pmb{v} )$,
        \item ${\Phi^\downarrow} \pmb{v}_\phi + {\Phi^\downarrow} \pmb{v}_\phi' = {\Phi^\downarrow} (\pmb{v}_\phi + \pmb{v}_\phi')$, and
        \item ${\Phi^\uparrow} \pmb{v} + {\Phi^\uparrow} \pmb{v}' = {\Phi^\uparrow} (\pmb{v} + \pmb{v}')$.
    \end{enumerate}
\end{proposition}
\begin{proof}
    The proof for each property follows by using algebra and~\cref{eq:down-projection-vec,eq:up-projection-vec}.
\end{proof}

\begin{lemma}
    \label{lemm:norm-and-projection}
    For a state abstraction $\phi: \natnum \to \statespace^\phi$ with a finite abstract state space,
    \begin{equation}
        \forall_{\pmb{v},\pmb{v}' \in \realnum^{\abs{\statespace^\phi}}}, ~ \forall_{\pmb{p} \in \probsimplex{\infty}}, ~ \norm{ {\Phi^\downarrow} \pmb{v} - {\Phi^\downarrow} \pmb{v}' }_{ \pmb{p} } = \norm{ \pmb{v} - \pmb{v}' }_{ {\Phi^\uparrow} \pmb{p} }.
    \end{equation}
\end{lemma}
\begin{proof}[Proof of~\cref{lemm:norm-and-projection}]
    Consider a state abstraction $\phi: \natnum \to \statespace^\phi$ with $\abs{ \statespace^\phi } \le \infty$, two vectors $\pmb{v},\pmb{v}' \in \realnum^{\abs{\statespace^\phi}}$ and a probability vector $\pmb{p} \in \probsimplex{\infty}$.
    Then,
    \begin{align}
        \norm{ {\Phi^\downarrow} \pmb{v} - {\Phi^\downarrow} \pmb{v}' }_{ \pmb{p} }
        &= \sum_{i=1}^\infty \pmb{p}(i) \abs{ ({\Phi^\downarrow} \pmb{v})(i) - ({\Phi^\downarrow} \pmb{v}')(i) } \\
        &= \sum_{i=1}^\infty \pmb{p}(i) \abs{  \pmb{v}(\phi(i)) -  \pmb{v}'(\phi(i)) } \\
        &= \sum_{s_\phi \in \statespace^\phi} \sum_{i | \phi(i) = s_\phi} \pmb{p}(i) \abs{  \pmb{v}(\phi(i)) -  \pmb{v}'(\phi(i)) } 
        = \norm{ \pmb{v}(s_\phi) -  \pmb{v}'(s_\phi) }_{ {\Phi^\uparrow} \pmb{p} }  .
    \end{align}
\end{proof}

\begin{lemma}[Up--down projection identity]
    \label{lemm:up-down-id}
    Let $\phi : \natnum \to \statespace_\phi$ be a state abstraction with $|\statespace_\phi| \le \infty$, and let $\mathbf{P} \in \realnum^{\infty \times \infty}$. 
    Then for every $i \in \natnum$ and every $\mathbf{x} \in \realnum^{|\statespace_\phi|}$,
    \begin{equation}
        \pmb{P}(i) ({\Phi^\downarrow} \pmb{x})   = ({\Phi^\uparrow}\pmb{P}(i)) \pmb{x}.
    \end{equation}
\end{lemma}
\begin{proof}[Proof of~\cref{lemm:up-down-id}]
    For all $\pmb{x}$, 
    \begin{align}
        \pmb{P}(i) ({\Phi^\downarrow} \pmb{x}) 
        = \sum_{j=1}^\infty \pmb{P}(i,j) \pmb{x} (\phi(j)) 
        = \sum_{s_\phi'} \underbrace{\sum_{j|\phi(j) = s_\phi} \pmb{P}(i,j)}_{={\Phi^\uparrow}\pmb{P}(i)} \pmb{x} (s_\phi') 
        &= \sum_{s_\phi'} ({\Phi^\uparrow}\pmb{P}(i))(s_\phi') \pmb{x} (s_\phi') \\
        &= ({\Phi^\uparrow}\pmb{P}(i)) \pmb{x}.   
    \end{align}
    where the last step follows from~\cref{eq:up-projection-vec}.
\end{proof}

\begin{lemma}[note={Generalization bound for natural number MDPs},store={lemm:gen-bound-natnum-mdp}]
    \label{lemm:gen-bound-natnum-mdp}
    For a state abstraction $\phi: \natnum \to \statespace_\phi$ with $\abs{ \statespace^\phi } < \infty$, consider a natural-number MDP $M^\natnum$ and an abstract MDP $M^\phi$.
    Let $\pmb{g} \in \realnum^{ \abs{\statespace^\phi \times \actionspace} }$ be the fixed point for $\bellmanop^{\pi_{\pmb{g}}}$ for this abstract MDP $M^\phi$.
    We assume that the abstract MDP $M^\phi$ need not be connected to the MDP $M^\natnum$. 
    Then, for any distribution vector $\pmb{d} \in \probsimplex{\infty}$,
    \begin{equation*}
        \norm{ {\Phi^{\downarrow}} \pmb{g} - \bellmanopinf^{\pi_{ {\Phi^{\downarrow}} \pmb{g} }}_{M^\natnum} {\Phi^{\downarrow}} \pmb{g} }_{ \pmb{d} }
        \le \norm{ {\Phi^\downarrow} \pmb{r}_{M^\phi} - \pmb{r}_{M^\natnum} }_{ \pmb{d} } + \frac{\gamma}{2} \vmax \norm{ {\Phi^\downarrow} \pmb{P}_{M^\phi} - {\Phi^\uparrow} \pmb{P}_{M^\natnum} }_{ \pmb{d} } .
    \end{equation*}
\end{lemma}
\begin{proof}[Proof of~\cref{lemm:gen-bound-natnum-mdp}]
    Let $\pmb{f} = {\Phi^{\downarrow}} \pmb{g}$ and from the definition of $\nu$, we observe
    \begin{equation}
        {\Phi^{\downarrow}} \nu( \pmb{g}, \pi_{ \pmb{g} } ) = \nu( {\Phi^{\downarrow}} \pmb{g}, \pi_{ {\Phi^{\downarrow}} \pmb{g} } ). \label{eq:nu-phi-proj-id}
    \end{equation}
    Then, for some state-action indexing bijection $k: \natnum \times \actionspace \to \natnum$,
    \begin{align}
        & ( {\Phi^{\downarrow}} \pmb{g} - \bellmanopinf^{\pi_{ {\Phi^{\downarrow}} \pmb{g} }}_{M^\natnum} {\Phi^{\downarrow}} \pmb{g} )(k(s,a)) \nonumber \\
        &= ({\Phi^{\downarrow}} \pmb{g})(k(s,a)) - \pmb{r}_{M^\natnum}(k(s,a)) - \gamma \pmb{P}_{M^\natnum}(k(s,a)) \nu( {\Phi^{\downarrow}} \pmb{g}, \pi_{ {\Phi^{\downarrow}} \pmb{g} } ) \nonumber \\
        &= \pmb{g}(\phi(s),a) - \pmb{r}_{M^\natnum}(k(s,a)) - \gamma \pmb{P}_{M^\natnum}(k(s,a)) \nu( {\Phi^{\downarrow}} \pmb{g}, \pi_{ {\Phi^{\downarrow}} \pmb{g} } ) \nonumber \\
        &= (\bellmanop_{M^\phi}^{\pi_{\pmb{g}}} \pmb{g})(\phi(s),a) - \pmb{r}_{M^\natnum}(k(s,a)) - \gamma \pmb{P}_{M^\natnum}(k(s,a)) \nu( {\Phi^{\downarrow}} \pmb{g}, \pi_{ {\Phi^{\downarrow}} \pmb{g} } ) ~~~~~~ (\text{by $\pmb{g} = \bellmanop^{\pi_{\pmb{g}}} \pmb{g}$}) \nonumber \\
        &= \pmb{r}_{M^\phi}(\phi(s),a) + \gamma \pmb{P}_{M^\phi}(\phi(s),a) \nu( \pmb{g}, \pi_{\pmb{g}} )  - \pmb{r}_{M^\natnum}(k(s,a)) - \gamma \pmb{P}_{M^\natnum}(k(s,a)) \nu( {\Phi^{\downarrow}} \pmb{g}, \pi_{ {\Phi^{\downarrow}} \pmb{g} } ) \nonumber \\
        &= \pmb{r}_{M^\phi}(\phi(s),a) - \pmb{r}_{M^\natnum}(k(s,a)) + \gamma \left[ \pmb{P}_{M^\phi}(\phi(s),a) \nu( \pmb{g}, \pi_{\pmb{g}} ) - \pmb{P}_{M^\natnum}(k(s,a)) \nu( {\Phi^{\downarrow}} \pmb{g}, \pi_{ {\Phi^{\downarrow}} \pmb{g} } ) \right]. \label{eq:gen-bound-natnum-mdp-proof-1} 
    \end{align}
    For the transition term in~\cref{eq:gen-bound-natnum-mdp-proof-1} we observe that
    \begin{align}
        & \pmb{P}_{M^\phi}(\phi(s),a) \nu( \pmb{g}, \pi_{\pmb{g}} ) - \pmb{P}_{M^\natnum}(k(s,a)) \nu( {\Phi^{\downarrow}} \pmb{g}, \pi_{ {\Phi^{\downarrow}} \pmb{g} } ) \nonumber \\
        &= \pmb{P}_{M^\phi}(\phi(s),a) \nu( \pmb{g}, \pi_{\pmb{g}} ) - \pmb{P}_{M^\natnum}(k(s,a)) {\Phi^{\downarrow}} \nu( {\Phi^{\downarrow}} \pmb{g}, \pi_{ \pmb{g} } ) . &(\text{by~\cref{eq:nu-phi-proj-id}})
    \end{align}
    By~\cref{lemm:up-down-id},
    \begin{equation}
        \forall \pmb{x} \in \realnum^{\abs{\statespace^\phi}}, ~ \pmb{P}_{M^\natnum}(k(s,a)) {\Phi^{\downarrow}} \pmb{x} = ( {\Phi^\uparrow} \pmb{P}_{M^\natnum}(k(s,a)) ) \pmb{x}.
    \end{equation}
    Therefore
    \begin{align}
        & ( {\Phi^{\downarrow}} \pmb{g} - \bellmanopinf^{\pi_{ {\Phi^{\downarrow}} \pmb{g} }}_{M^\natnum} {\Phi^{\downarrow}} \pmb{g} )(k(s,a)) \nonumber \\
        &= \pmb{r}_{M^\phi}(\phi(s),a) - \pmb{r}_{M^\natnum}(k(s,a)) + \gamma \left[ \pmb{P}_{M^\phi}(\phi(s),a) \nu( \pmb{g}, \pi_{\pmb{g}} )  - ( {\Phi^\uparrow} \pmb{P}_{M^\natnum}(k(s,a)) ) \nu( \pmb{g}, \pi_{ \pmb{g} } ) \right] \nonumber \\
        &= \pmb{r}_{M^\phi}(\phi(s),a) - \pmb{r}_{M^\natnum}(k(s,a)) + \gamma \left( \pmb{P}_{M^\phi}(\phi(s),a) - {\Phi^\uparrow} \pmb{P}_{M^\natnum}(k(s,a))  \right) \nu( \pmb{g}, \pi_{ \pmb{g} } ) \nonumber \\
        &= \pmb{r}_{M^\phi}(\phi(s),a) - \pmb{r}_{M^\natnum}(k(s,a)) + \gamma \left( ({\Phi^\downarrow} \pmb{P}_{M^\phi})(k(s,a)) - {\Phi^\uparrow} \pmb{P}_{M^\natnum}(k(s,a))  \right) \nu( \pmb{g}, \pi_{ \pmb{g} } ) \nonumber \\
        &= ({\Phi^\downarrow} \pmb{r}_{M^\phi} - \pmb{r}_{M^\natnum})(k(s,a)) + \gamma \left( {\Phi^\downarrow} \pmb{P}_{M^\phi} - {\Phi^\uparrow} \pmb{P}_{M^\natnum} \right)(k(s,a)) \nu( \pmb{g}, \pi_{ \pmb{g} } ) \label{eq:pre-vmax-half-trick}
    \end{align}
    Note that the term $\left( {\Phi^\downarrow} \pmb{P}_{M^\phi} - {\Phi^\uparrow} \pmb{P}_{M^\natnum} \right)(k(s,a)) \nu( \pmb{g}, \pi_{ \pmb{g} } )$ computes a dot product between to infinite-dimensional vectors.
    Thus we can write
    \begin{align}
        & \left( {\Phi^\downarrow} \pmb{P}_{M^\phi} - {\Phi^\uparrow} \pmb{P}_{M^\natnum} \right)(k(s,a)) \left( \nu( \pmb{g}, \pi_{ \pmb{g} } ) - \frac{\vmax}{2} \pmb{1} \right) \nonumber \\
        &= \sum_{j=1}^\infty \left( {\Phi^\downarrow} \pmb{P}_{M^\phi} - {\Phi^\uparrow} \pmb{P}_{M^\natnum} \right)(k(s,a),j) \left( \nu( \pmb{g}, \pi_{ \pmb{g} } )(j) - \frac{\vmax}{2} \pmb{1} \right) \nonumber \\
        &= \sum_{j=1}^\infty \left( {\Phi^\downarrow} \pmb{P}_{M^\phi} - {\Phi^\uparrow} \pmb{P}_{M^\natnum} \right)(k(s,a),j) \nu( \pmb{g}, \pi_{ \pmb{g} } )(j) - 
        \underbrace{
            \left( \sum_{j=1}^\infty \left( {\Phi^\downarrow} \pmb{P}_{M^\phi} - {\Phi^\uparrow} \pmb{P}_{M^\natnum} \right)(k(s,a),j) \right)
        }_{=0} 
        \frac{\vmax}{2} \pmb{1} \nonumber \\
        &= \sum_{j=1}^\infty \left( {\Phi^\downarrow} \pmb{P}_{M^\phi} - {\Phi^\uparrow} \pmb{P}_{M^\natnum} \right)(k(s,a),j) \nu( \pmb{g}, \pi_{ \pmb{g} } )(j) \label{eq:vmax-half-trick}
    \end{align}
    Substituting~\cref{eq:vmax-half-trick} into~\cref{eq:pre-vmax-half-trick} results in the following bound:
    \begin{align}
        & ( {\Phi^{\downarrow}} \pmb{g} - \bellmanopinf^{\pi_{ {\Phi^{\downarrow}} \pmb{g} }}_{M^\natnum} {\Phi^{\downarrow}} \pmb{g} )(k(s,a)) \nonumber \\
        &= ({\Phi^\downarrow} \pmb{r}_{M^\phi} - \pmb{r}_{M^\natnum})(k(s,a)) + \gamma \left( {\Phi^\downarrow} \pmb{P}_{M^\phi} - {\Phi^\uparrow} \pmb{P}_{M^\natnum} \right)(k(s,a)) \left( \nu( \pmb{g}, \pi_{ \pmb{g} } ) - \frac{\vmax}{2} \pmb{1} \right) . 
    \end{align}
    For $i = k(s,a)$ we have
    \begin{align*}
        & \abs{( {\Phi^{\downarrow}} \pmb{g} - \bellmanopinf^{\pi_{ {\Phi^{\downarrow}} \pmb{g} }}_{M^\natnum} {\Phi^{\downarrow}} \pmb{g} )(i)} \\
        &\le \abs{({\Phi^\downarrow} \pmb{r}_{M^\phi} - \pmb{r}_{M^\natnum})(i)} + \gamma \norm{ ({\Phi^\downarrow} \pmb{P}_{M^\phi} - {\Phi^\uparrow} \pmb{P}_{M^\natnum})(i) }_1 \norm{ \nu( \pmb{g}, \pi_{ \pmb{g} } ) - \frac{\vmax}{2} \pmb{1} }_\infty \\
        &\le \abs{({\Phi^\downarrow} \pmb{r}_{M^\phi} - \pmb{r}_{M^\natnum})(i)} + \gamma \norm{ ({\Phi^\downarrow} \pmb{P}_{M^\phi} - {\Phi^\uparrow} \pmb{P}_{M^\natnum})(i) }_1 \frac{\vmax}{2}
    \end{align*}
    Picking a distribution over all state action pairs $\pmb{d} \in \probsimplex{ \infty }$ we have
    \begin{equation*}
        \norm{ {\Phi^{\downarrow}} \pmb{g} - \bellmanopinf^{\pi_{ {\Phi^{\downarrow}} \pmb{g} }}_{M^\natnum} {\Phi^{\downarrow}} \pmb{g} }_{ \pmb{d} }
        \le \norm{ {\Phi^\downarrow} \pmb{r}_{M^\phi} - \pmb{r}_{M^\natnum} }_{ \pmb{d} } + \frac{\gamma}{2} \vmax \norm{ {\Phi^\downarrow} \pmb{P}_{M^\phi} - {\Phi^\uparrow} \pmb{P}_{M^\natnum} }_{ \pmb{d} } .
    \end{equation*}
\end{proof}

\noindent \textbf{Restatment of~\cref{lemm:performance-loss-succ-weighted-model-reduction}.}
    \getkeytheorem[body]{lemm:performance-loss-succ-weighted-model-reduction}
\begin{proof}[Proof of~\cref{lemm:performance-loss-succ-weighted-model-reduction}]
    The proof is an application of the bound in~\cref{lemm:gen-bound-natnum-mdp} to the bound in~\cref{lemm:policy-perf-diff} by setting $\pmb{f} = {\Phi^\downarrow} \pmb{g}$ for $\pmb{g} \in \realnum^{\abs{ \statespace^\phi \times \actionspace }}$ and the using the identity
    \begin{equation*}
        \norm{ \pmb{f} - \bellmanopinf^{\pi_{\pmb{f}}}_{M^\natnum} \pmb{f} }_{\pmb{d}^\pi}
        = \norm{ {\Phi^{\downarrow}} \pmb{g} - \bellmanopinf^{\pi_{ {\Phi^{\downarrow}} \pmb{g} }}_{M^\natnum} {\Phi^{\downarrow}} \pmb{g} }_{ \pmb{d} } 
    \end{equation*}
    for the respective distributions $\pmb{d}$.
\end{proof}

\noindent \textbf{Restatement of~\cref{thm:ood-pomdp}.}
    \getkeytheorem[body]{thm:ood-pomdp}
\begin{proof}[Proof of~\cref{thm:ood-pomdp}]
    By~\cref{prop:infinite-abstract-projection-associativity}, we have that
    \begin{eqnarray}
        {\Phi^\downarrow} \pmb{r}_{M^\phi_\train} = {\Phi^\downarrow} \pmb{r}_{M^\phi_\test} + {\Phi^\downarrow} \pmb{\eps}_r^\ood, & & {\Phi^\downarrow} \pmb{P}_{M^\phi_\train} = {\Phi^\downarrow} \pmb{P}_{M^\phi_\test} + {\Phi^\downarrow} \pmb{\eps}_p^\ood .
    \end{eqnarray}
    Therefore,
    \begin{align*}
        \norm{ {\Phi^\downarrow} \pmb{r}_{M_\train^\phi} - \pmb{r}_{M_\test^\natnum} }_{ \pmb{d} } 
        &= \norm{ {\Phi^\downarrow} ( \pmb{r}_{M^\phi_\test} + \pmb{\eps}_r^\ood ) - \pmb{r}_{M_\test^\natnum} }_{ \pmb{d} } \\
        &\le \norm{ {\Phi^\downarrow} \pmb{r}_{M^\phi_\test} - \pmb{r}_{M_\test^\natnum} }_{ \pmb{d} } + \norm{ {\Phi^\downarrow} \pmb{\eps}_r^\ood }_{ \pmb{d} } \\ 
        &= \norm{ {\Phi^\downarrow} \pmb{r}_{M^\phi_\test} - \pmb{r}_{M_\test^\natnum} }_{ \pmb{d} } + \norm{ \pmb{\eps}_r^\ood }_{ {\Phi^\uparrow} \pmb{d} } &(\text{by~\cref{lemm:norm-and-projection}}) \\
        &\le \norm{ \pmb{\eps}_r^\phi }_{ \pmb{d} } + \norm{ \pmb{\eps}_r^\ood }_{ {\Phi^\uparrow} \pmb{d} }
    \end{align*}
    where $\pmb{d} \in \{ \pmb{d}^{\pi^\star}, \pmb{d}^{\pi^\star_{M^\phi_\train}} \}$.
    Similarly, 
    \begin{align*}
        \norm{ {\Phi^\downarrow} \pmb{P}_{M_\train^\phi} - {\Phi^\uparrow} \pmb{P}_{M_\test^\natnum} }_{ \pmb{d} } 
        &= \norm{ {\Phi^\downarrow} (\pmb{P}_{M_\test^\phi} + \pmb{\eps}_p^\ood) - {\Phi^\uparrow} \pmb{P}_{M_\test^\natnum} }_{ \pmb{d} } \\
        &\le \norm{ {\Phi^\downarrow} \pmb{P}_{M_\test^\phi} - {\Phi^\uparrow} \pmb{P}_{M_\test^\natnum}  }_{ \pmb{d} } + \norm{ {\Phi^\downarrow} \pmb{\eps}_p^\ood }_{ \pmb{d} } \\
        &= \norm{ {\Phi^\downarrow} \pmb{P}_{M_\test^\phi} - {\Phi^\uparrow} \pmb{P}_{M_\test^\natnum}  }_{ \pmb{d} } + \norm{ \pmb{\eps}_p^\ood }_{ {\Phi^\uparrow} \pmb{d} } &(\text{by~\cref{lemm:norm-and-projection}}) \\
        &\le \norm{ \pmb{\eps}_p^\phi }_{ \pmb{d} } + \norm{ \pmb{\eps}_p^\ood }_{ {\Phi^\uparrow} \pmb{d} }.
    \end{align*}
    Putting these results together gives us
    \begin{align*}
        \norm{ {\Phi^{\downarrow}} \pmb{g} - \bellmanopinf^{\pi_{ {\Phi^{\downarrow}} \pmb{g} }}_{M^\natnum} {\Phi^{\downarrow}} \pmb{g} }_{ \pmb{d} }
        \le \norm{ \pmb{\eps}_r^\phi }_{ \pmb{d} } + \norm{ \pmb{\eps}_r^\ood }_{ {\Phi^\uparrow} \pmb{d} } + \frac{\gamma}{2} \vmax \norm{ \pmb{\eps}_r^\phi }_{ \pmb{d} } + \frac{\gamma}{2} \vmax \norm{ \pmb{\eps}_p^\ood }_{ {\Phi^\uparrow} \pmb{d} }
    \end{align*}
    Putting these results together with~\cref{lemm:policy-perf-diff} for $\pmb{f} = {\Phi^\downarrow} \pmb{g}$ where $\pi_{ \pmb{f} } = \pi^\star_{M^\phi_\train}$ and a natural number MDP $M^\natnum_\test$ gives us
    \begin{align*}
        & J(\pi^\star, \pmb{d}_\start^\test) - J(\pi^\star_{M^\phi_\train}, \pmb{d}_\start^\test) \\
        & \le \frac{1}{1 - \gamma}
        \Big(
              \norm{ {\phi^\downarrow} \pmb{g} - \bellmanopinf^{\pi_{ {\Phi^{\downarrow}} \pmb{g} }}_{M^\natnum} {\phi^\downarrow} \pmb{g} }_{ \pmb{d}^{\pi^\star} }
            + \norm{ {\phi^\downarrow} \pmb{g} - \bellmanopinf^{\pi_{ {\Phi^{\downarrow}} \pmb{g} }}_{M^\natnum} {\phi^\downarrow} \pmb{g} }_{ \pmb{d}^{\pi^\star_{M^\phi_\train}} }
        \Big) \\
        & \le \frac{1}{1 - \gamma} 
            \sum_{ \pmb{d} \in \{ {\Phi^\uparrow} \pmb{d}^{\pi^\star_{M^\natnum_\test}} , {\Phi^\uparrow} \pmb{d}^{\pi^\star_{M^\phi_\train}} \} } \norm{ \pmb{\eps}_r^\phi }_{ \pmb{d} } + \frac{\gamma}{2} \vmax \norm{ \pmb{\eps}_p^\phi }_{ \pmb{d} } + \norm{ \pmb{\eps}_r^\ood }_{ \pmb{d} } + \frac{\gamma}{2} \vmax \norm{ \pmb{\eps}_p^\ood }_{ \pmb{d} } 
    \end{align*}
    Note that $\pi^\star_{M^\phi_\train} = \pi^\phi$ is the policy learned under the training start-state distribution using state abstraction $\phi$.
    Because the natural-number MDP is bisimilar to the history MDP induced under the corresponding start distribution, the same bound transfers to the POMDP $P$ setting for start-state distribution $\pmb{d}_\start^\test$.
\end{proof}

\subsection{Proof of corollary}
\label{app:proof-of-corollary}

\subsubsection{Approximation errors decrease with the abstract state space size}
\label{app:approx-error-abs-state-decrease}

The approximation error (\cref{thm:ood-pomdp,lemm:performance-loss-succ-weighted-model-reduction}) can be bounded with the number of abstract states by observing the following property of distribution weighted norms.
For example, consider the infinite-dimensional vector $\pmb{v}$ and some infinite-dimensional probability distribution vector $\pmb{d} \in \probsimplex{\infty}$.
If the weighted norm $\norm{ \pmb{v} }_{ \pmb{d} }$ is small, then entries for which $\pmb{d}(i)$ is large the absolute value of the corresponding vector entries $\pmb{v}(i)$ is small.
Similarly, entries $\pmb{v}(i)$ that are high in absolute magnitude, will be de-emphasized by $\pmb{d}(i)$.
Assuming that every entry of $\pmb{v}$ is bounded in magnitude by $C$, we can write for some index subset $\mathcal{I} \subset \natnum$ that
\begin{align}
    \norm{ \pmb{v} }_{ \pmb{d} } 
    = \sum_{i \in \natnum} \pmb{d}(i) \abs{ \pmb{v}(i) } 
    &= \sum_{i \in \mathcal{I}} \pmb{d}(i) \abs{ \pmb{v}(i) } + \sum_{i \not \in \mathcal{I}} \pmb{d}(i) \abs{ \pmb{v}(i) } \\
    &\le \max_{ i \in \mathcal{I} } \abs{ \pmb{v}(i) } \underbrace{\sum_{i \in \mathcal{I}} \pmb{d}(i)}_{=1-\delta} + C \underbrace{\sum_{i \not \in \mathcal{I}} \pmb{d}(i)}_{=\delta} 
    = (1 - \delta) \max_{ i \in \mathcal{I} } \abs{ \pmb{v}(i) } + \delta C. \label{eq:maximum-norm-connection}
\end{align}
\Cref{eq:maximum-norm-connection} shows that a maximum norm computed over the index set $\mathcal{I}$ can be used to upper-bound the distribution-weighted norm over the full vector $\pmb{v}$.
Using this property, the following lemma bounds the reward and transition approximation errors of a successor-weighted model reduction (\cref{def:successor-weighted-model-reduction}) with the abstract state space size $\abs{\statespace^\phi}$.

\begin{lemma}[Approximation errors decrease with abstract state space size]
    \label{lemm:approx-error-decrease-abstract-states}
    For any natural number MDP and any distribution vector $\pmb{d}$, the approximation errors (\cref{def:successor-weighted-model-reduction}) are bounded by
    \begin{eqnarray*}
        \norm{ \pmb{r} - {\Phi^\downarrow} \pmb{r}_\phi }_{ \pmb{d} } = O \Big( \abs{\statespace^\phi}^{- \frac{1}{\abs{\actionspace}}} \Big)
        & \text{and} &
        \norm{ {\Phi^\uparrow} \pmb{P} - {\Phi^\downarrow} \pmb{P}_\phi }_{ \pmb{d} } = O \Big( \abs{\statespace^\phi}^{- \frac{1}{\abs{\statespace^\phi}\abs{\actionspace}} } \Big).
    \end{eqnarray*}
\end{lemma}
\begin{proof}[Proof of~\cref{lemm:approx-error-decrease-abstract-states}]
    We consider the reward and transition approximation error bound separately.
    
    \textbf{Reward approximation error bound.}
    Using the strategy outlined in~\cref{eq:maximum-norm-connection}, we consider a set $\mathcal{I} \subset \natnum$ such that $\delta$ is small.
    Let 
    \begin{equation}
        \eps_r^\phi = \max_{i \in \mathcal{I}} \abs{ (\pmb{r} - {\Phi^\downarrow} \pmb{r}_\phi)(i) }.
    \end{equation}
    Consider any state $s$ for which the reward approximation error is at most $\eps_r^\phi$.
    By construction of the set $\mathcal{I}$, $k(s,a) \in \mathcal{I}$.
    Because every reward value falls into the interval $[0,\rmax]$, we can partition this interval into $\lceil \frac{1}{\eps_r^\phi} \rceil$ intervals.
    We can construct a state abstraction by taking every state that belongs to the index set $\mathcal{I}$, write out the one-step rewards associated with this state as a vector, and then bin this vector into one of the $\lceil \frac{1}{\eps_r^\phi} \rceil^\abs{\actionspace}$ many different reward partitions.
    Associating each abstract state with one of these partitions, we observe that 
    \begin{equation*}
        |\statespace^\phi| \le (\frac{1}{\eps_r^\phi})^{|\actionspace|}
        \iff 
        \eps_r^\phi \le |\statespace^\phi|^{-\frac{1}{{|\actionspace|}}} .
    \end{equation*}
    Therefore, we apply~\cref{eq:maximum-norm-connection} and obtain
    \begin{equation}
        \norm{ \pmb{r} - {\Phi^\downarrow} \pmb{r}_\phi }_{ \pmb{d} } \le (1 - \delta_r) |\statespace^\phi|^{-\frac{1}{{|\actionspace|}}} + \delta_r \rmax
    \end{equation}
    and thus the reward approximation error is in $O \left( \abs{\statespace^\phi}^{- \frac{1}{\abs{\actionspace}}} \right)$.

    \textbf{Transition approximation error bound.}
    Similarly, we consider a set $\mathcal{I} \subset \natnum$ such that $\delta$ is small and let
    \begin{equation}
        \eps_p^\phi = \max_{i \in \mathcal{I}} \norm{ ( {\Phi^\uparrow} \pmb{P} - {\Phi^\downarrow} \pmb{P}_\phi )(i) }_1.
    \end{equation}
    Because the probability vectors have entries that fall into the range $[0,1]$, we can partition this interval into $\lceil \frac{1}{\eps_p^\phi} \rceil$ many segments.
    Because the rows are vectors of dimension $\abs{ \statespace^\phi }$, the number of distinct transition vectors for a fixed action is bounded by $\lceil \frac{1}{\eps_p^\phi}\rceil^{|\statespace^\phi|}$. 
    Then for all actions in $\actionspace$, the total number of distinct transition  is bounded by $\lceil\frac{1}{\eps_p^\phi} \rceil^{|\statespace^\phi||\actionspace|}$. 
    Therefore, the number of abstract states satisfies
    \begin{equation*}
        \lceil \frac{1}{\eps_p^\phi}\rceil^{|\statespace^\phi||\actionspace|} \ge |\statespace^\phi|
        \iff
        {\eps_p^\phi} \le |\statespace^\phi|^{-\frac{1}{  {|\statespace^\phi||\actionspace|} }}.
    \end{equation*}
    By~\cref{eq:maximum-norm-connection} we obtain similar to before that
    \begin{equation}
        \norm{ {\Phi^\uparrow} \pmb{P} - {\Phi^\downarrow} \pmb{P}_\phi }_{ \pmb{d} } \le (1 - \delta_p) |\statespace^\phi|^{-\frac{1}{  {|\statespace^\phi||\actionspace|} }} + 2 \delta_p 
    \end{equation}
    and thus the transition approximation error is in $O \left( \abs{\statespace^\phi}^{- \frac{1}{\abs{\statespace^\phi} \abs{\actionspace}}} \right)$.

    To obtain the final result, these two individual results can be combined into one state abstraction by finding an integer subset $\mathcal{I}$ such that both $\delta_r$, $\delta_p$, $\eps_r^\phi$ and $\eps_p^\phi$ are small.
\end{proof}

\subsection{OOD estimation error bound}
\label{app:ood-estimation-error-bound}

A Dirichlet distribution with parameters $\alpha_1,...,\alpha_n$ is a density function over the $n$-dimensional probability simplex.
Specifically, the probability of sampling a distribution vector $\pmb{p} \sim \text{Dir}(1,...,1)$ is uniform across the set of possible probability vectors $\probsimplex{n}$.
Moreover, for every entry $\pmb{p}(i)$ of this vector follows a $\text{Beta}(1, n-1)$ distribution with mean $\frac{1}{n}$.

First we prove the following lemmas before proving~\cref{corl:state-space-dependency}.
    
\begin{lemma}
    \label{lemm:subset-prob}
    Let $\pmb{p} \in \Delta(n)$ be a probability vector that is sampled uniformly at random from a flat Dirichlet distribution.
    Consider an index subset $\mathcal{I} \subset \{ 1,...,n \}$ of size $k$.
    Then,
    \begin{equation}
        \expec{\pmb{p} \sim \text{Dir}(1,...,1)}{\pmb{p}(\mathcal{I})} = \frac{k}{n}.
    \end{equation}
\end{lemma}
\begin{proof}
    For an arbitrary index set $\mathcal{I}$, we have that
    \begin{align}
        \expec{\pmb{p} \sim \text{Dir}(1,...,1)}{\pmb{p}(\mathcal{I})} 
        = \expec{\pmb{p} \sim \text{Dir}(1,...,1)}{\sum_{i \in \mathcal{I}} \pmb{p}_i} 
        = \sum_{i \in \mathcal{I}} \expec{\pmb{p} \sim \text{Dir}(1,...,1)}{\pmb{p}_i} 
        = \sum_{i \in \mathcal{I}} \frac{1}{n} 
        = \frac{k}{n}.
    \end{align}
\end{proof} 

\begin{lemma}
    \label{lemm:unknown-indices}
    Suppose we draw indices $1,...,n$ from a probability vector $\pmb{p}$ with replacement $T$ times.
    Let $\mathcal{I}$ be the set of indices not observed in the final sample.
    Then,
    \begin{equation}
        \expec{\mathcal{I} \sim \pmb{p}}{ \abs{\mathcal{I}} } = \sum_{i=1}^n (1 - \pmb{p}_i)^T.
    \end{equation}
\end{lemma}
\begin{proof}
    First, we define an indicator random variable $X_i$ that takes the value one if index $i$ is not sampled and takes the value zero if the index $i$ is sampled.
    Then,
    \begin{equation}
        \prob{\text{$i$ not sampled in $T$ draws}} = \expec{X_i}{  X_i } = (1 - \pmb{p}_i)^T.
    \end{equation}
    Furthermore, we have that
    \begin{align*}
        \expec{\mathcal{I} \sim \pmb{p}}{ \abs{\mathcal{I}} } 
        = \expec{\mathcal{I} \sim \pmb{p}}{ \sum_{i = 1}^n X_i } 
        = \sum_{i = 1}^n \prob{\text{$i$ not sampled in $T$ draws}}
        = \sum_{i = 1}^n (1 - \pmb{p}_i)^T.
    \end{align*}
\end{proof}

\begin{lemma}
    \label{lemm:unknown-indices-sampled}
    Suppose we sample a probability vector $\pmb{p} \in \Delta(n)$ uniformly at random from a flat Dirichlet distribution. 
    Then, we sample indices from this vector $T$ times with replacement.
    Then,
    \begin{equation}
        \expec{\pmb{p} \sim \text{Dir}(1,...,1), \mathcal{I} \sim \pmb{p}}{ \abs{\mathcal{I}} } = \frac{n(n-1)}{T+n-1}.
    \end{equation}
\end{lemma}
\begin{proof}
    First, we observe that
    \begin{align}
        \expec{\pmb{p}_i \sim \text{Beta}(1, n-1)}{(1 - \pmb{p}_i)^T} 
        &= \int_0^1 \prob{\pmb{p}_i} (1 - \pmb{p}_i)^T \text{d} \pmb{p}_i \\
        &= \int_0^1 (n-1)(1 - \pmb{p}_i)^{n-2} (1 - \pmb{p}_i)^T \text{d} \pmb{p}_i 
        =\frac{n-1}{T+n-1}. \label{eq:expec-pi-T}
    \end{align}
    \Cref{eq:expec-pi-T} is obtained using the identity $\int_0^1 (1 - p)^a dp = \frac{1}{a+1}$ for $a > 0$ and $p \in [0,1]$.
    By~\cref{eq:expec-pi-T,lemm:unknown-indices} we obtain
    \begin{equation}
        \expec{\pmb{p} \sim \text{Dir}(1,...,1), \mathcal{I} \sim \pmb{p}}{ \abs{\mathcal{I}} } = \sum_{i=1}^n \expec{}{(1 - \pmb{p}_i)^T} = \frac{n(n-1)}{T+n-1}.
    \end{equation}
\end{proof}

\begin{lemma}
    \label{lemm:uniform-prior-lemma}
    Suppose we sample training and test distribution vectors $\pmb{d}^\train$ and $\pmb{d}^\test$ sampled uniformly at random using a flat Dirichlet distribution.
    From the training distribution, we sample $T$ times a subset of known states $\statespace^\phi_{\text{knw}}$.
    Then,
    \begin{equation*}
        \expec{ \statespace^\phi_{\text{knw}},\pmb{d}^\test }{
            \pmb{d}^\test(\statespace^\phi_{\text{knw}}) \eps \vmax 
            + \pmb{d}^\test(\statespace^\phi_{\text{unk}}) \vmax
        } = \left( 1 - S \right) \eps \vmax + S \vmax
    \end{equation*}
    where $S = \frac{\abs{ \statespace^\phi }-1}{T+\abs{ \statespace^\phi }-1}$.
\end{lemma}
\begin{proof}
    \begin{align}
        & \expec{ \statespace^\phi_{\text{unk}},\pmb{d}^\test }{
            \pmb{d}^\test(\statespace^\phi_{\text{knw}}) \eps \vmax
            + \pmb{d}^\test(\statespace^\phi_{\text{unk}}) \vmax
        } \\
        &= \frac{ \expec{\statespace^\phi_{\text{unk}}}{ \abs{ \statespace^\phi_{\text{knw}} } } }{ \abs{ \statespace^\phi } } \eps \vmax
        + \frac{ \expec{\statespace^\phi_{\text{unk}}}{ \abs{ \statespace^\phi_{\text{unk}} } } }{ \abs{ \statespace^\phi } } \vmax &(\text{by~\cref{lemm:subset-prob}})\\
        &= \left( 1 - \frac{\abs{ \statespace^\phi }-1}{T+\abs{ \statespace^\phi }-1} \right) \eps \vmax
        + \frac{\abs{ \statespace^\phi }-1}{T+\abs{ \statespace^\phi }-1} \vmax &(\text{by~\cref{lemm:unknown-indices-sampled}})
    \end{align}
\end{proof}

\subsection{Proof of~\cref{corl:state-space-dependency}}

\textbf{Restatement of~\cref{corl:state-space-dependency}.}
    \getkeytheorem[body]{corl:state-space-dependency}
\begin{proof}[Proof of~\cref{corl:state-space-dependency}]
    By~\cref{thm:ood-pomdp,lemm:approx-error-decrease-abstract-states} the approximation error is bounded by a term $O \left( \frac{1}{\abs{\statespace^\phi}^{\abs{\actionspace}}} \right)$.

    The OOD estimation error is bounded by an estimation error between the learned and ground truth training MDP and the difference between the training and test abstract MDPs:
    \begin{align*}
        \norm{ \widehat{\pmb{r}}_{M_\train^\phi} - \pmb{r}_{M_\test^\phi} }_{\pmb{d}} &\le 
        \norm{ \widehat{\pmb{r}}_{M_\train^\phi} - \pmb{r}_{M_\train^\phi} }_{\pmb{d}} + \norm{ \pmb{\eps}_r^\ood }_{\pmb{d}} \\
        \norm{ \widehat{\pmb{P}}_{M_\train^\phi} - \pmb{P}_{M_\test^\phi} }_{\pmb{d}} &\le 
         \norm{ \widehat{\pmb{P}}_{M_\train^\phi} - \pmb{P}_{M_\train^\phi} }_{\pmb{d}} + \norm{ \pmb{\eps}_p^\ood }_{\pmb{d}}
    \end{align*}
    These bound follow by triangle inequality.
    By~\cref{asmpt:learning-finite-steps-abstraction}, we have that
    \begin{align*}
        \norm{ \widehat{\pmb{r}}_{M_\train^\phi} - \pmb{r}_{M_\train^\phi} }_{\pmb{d}} 
            &= \pmb{d}(\statespace^\phi_{\text{knw}}) \eps 
            + \pmb{d}(\statespace^\phi_{\text{unk}}) \rmax, \\
        \norm{ \widehat{\pmb{P}}_{M_\train^\phi} - \pmb{P}_{M_\train^\phi} }_{\pmb{d}} 
            &= \pmb{d}(\statespace^\phi_{\text{knw}}) \eps 
            + \pmb{d}(\statespace^\phi_{\text{unk}}) 2 .
    \end{align*}
    Thus, we have that
    \begin{align*}
        & \norm{ \widehat{\pmb{r}}_{M_\train^\phi} - \pmb{r}_{M_\test^\phi} }_{\pmb{d}} + \frac{\gamma}{2} \vmax \norm{ \widehat{\pmb{P}}_{M_\train^\phi} - \pmb{P}_{M_\test^\phi} }_{\pmb{d}} \\
        &= \pmb{d}(\statespace^\phi_{\text{knw}}) \eps 
            + \pmb{d}(\statespace^\phi_{\text{unk}}) \rmax
            + \norm{ \pmb{\eps}_r^\ood }_{\pmb{d}}
            + \frac{\gamma}{2} \vmax \pmb{d}(\statespace^\phi_{\text{knw}}) \eps 
            + \frac{\gamma}{2} \vmax \pmb{d}(\statespace^\phi_{\text{unk}}) 2 
            + \frac{\gamma}{2} \vmax \norm{ \pmb{\eps}_p^\ood }_{\pmb{d}} \\
        &\le \pmb{d}(\statespace^\phi_{\text{unk}}) \eps \vmax + \pmb{d}(\statespace^\phi_{\text{unk}}) \vmax + \norm{ \pmb{\eps}_r^\ood }_{\pmb{d}} + \frac{\gamma}{2} \vmax \norm{ \pmb{\eps}_p^\ood }_{\pmb{d}}
    \end{align*}
    Let
    \begin{equation}
        B = \max_{\pmb{d} \in \{ \pmb{d}^{\pi^\star} , \pmb{d}^{\pi^\phi} \} } \norm{ \pmb{\eps}_r^\ood }_{\pmb{d}} + \frac{\gamma}{2} \vmax \norm{ \pmb{\eps}_p^\ood }_{\pmb{d}} .
    \end{equation}
    Then, by~\cref{thm:ood-pomdp,lemm:uniform-prior-lemma}, we have that 
    \begin{equation*}
        J(\pi^\star, \pmb{d}_\start^\test) - J(\pi^\phi, \pmb{d}_\start^\test) 
        \le O \left( \frac{1}{\abs{\statespace^\phi}^{\abs{\actionspace}}}
        + \left( 1 - \frac{\abs{ \statespace^\phi }-1}{T+\abs{ \statespace^\phi }-1} \right) \eps
        + \frac{\abs{ \statespace^\phi }-1}{T+\abs{ \statespace^\phi }-1} + B \right),
    \end{equation*}
    while ignoring the dependency on $\gamma$, $\vmax$, and a factor 2.
\end{proof}

\section{Implementation details of tabular generalization experiments}
\label{app:experiments}

\paragraph{Warm--cold example}
The dictionary is constructed on a training lattice of radius $r=3$ by exhaustively enumerating all valid start states (\cref{fig:warm-cold-start-states}) and action sequences up to history lengths $k$, with experiments evaluating different history suffix lengths $k$.
To evaluate generalization across scale, testing is performed on both the training lattice and larger lattices (for example start distances of 10, 50, or 100), where test states are sampled exhaustively from the lattice boundary (outer diamond in~\cref{fig:warm-cold-start-states}). 
During testing, the optimal action counts stored in the policy dictionary are normalized to compute an action selection probability distribution.
If during testing no history suffix is found in the dictionary, actions are selected uniformly at random.
For each test start state, five independent walks are performed. 
Episodes terminate upon reaching the goal or after a maximum of 500 steps.

\paragraph{Sign-chain example}
In the sign chain example, the no abstraction method simulates every possible action sequence starting from each training start state.
The resulting histories (at each time step) are then stored in the policy dictionary together with a count of which action would be optimal at the corresponding ground state.
These counts are then used to compute an action selection probability distribution for selecting actions during testing.
If during testing a history is not found, actions are selected uniformly at random.
The same procedure is used for the optimal abstraction except that only the first observation is preserved when inserting a key into the policy dictionary.

\bigskip
\bibliography{library}

\end{document}